\documentclass[11pt]{article}
\usepackage{geometry}
\usepackage{amsmath,amsfonts,amsthm,amssymb,url}

\usepackage[dvips]{graphicx}  
\usepackage[utf8]{inputenc} 
\usepackage[T1]{fontenc}    
\usepackage{url}            
\usepackage[usenames,dvipsnames,svgnames,table]{xcolor}
\usepackage[colorlinks=true, linkcolor=red, urlcolor=blue, citecolor=gray]{hyperref}
\usepackage{booktabs}       
\usepackage{amsfonts}       
\usepackage{nicefrac}       
\usepackage{microtype}      
\usepackage{multicol}
\usepackage{float} 
\usepackage{enumitem}
\usepackage[square,sort,comma,numbers]{natbib}

\usepackage{amsmath}
\usepackage{amssymb}
\usepackage{amsthm}
\usepackage{xcolor}
\usepackage{footnote}
\usepackage{multirow}

\usepackage{subcaption}

\newtheorem{theorem}{Theorem}
\newtheorem{lemma}{Lemma}
\newtheorem{proposition}{Proposition}
\newtheorem{corollary}{Corollary}
\newtheorem{observation}{Observation}

\newtheorem*{mainres*}{Main results}

\theoremstyle{definition}
\newtheorem{definition}{Definition}
\newtheorem{example}{Example}

\definecolor{darkgreen}{RGB}{0,150,0}
\definecolor{darkorange}{RGB}{140,90,10}
\definecolor{darkred}{RGB}{150,0,0}

\usepackage[suppress]{color-edits}  
\addauthor{nh}{darkred} 
\addauthor{bo}{darkorange} 
\addauthor{Cam}{blue} 
\addauthor{adam}{darkgreen}

\DeclareMathOperator*{\E}{\mathbb{E}}
\DeclareMathOperator*{\argmin}{argmin}
\DeclareMathOperator*{\argmax}{argmax}

\newcommand{\eqdef}{{:=}}
\newcommand{\Ind}[1]{\mathbf{1}\left[#1\right]}  

\newcommand{\R}{\mathbb{R}}    
\newcommand{\norm}[1]{\left\|#1\right\|}

\newcommand{\TV}{\text{TV}}  
\newcommand{\KL}{\text{KL}}  
\newcommand{\poly}{\mathrm{poly}}

\newcommand{\X}{\mathcal{X}}  
\renewcommand{\P}{\mathcal{P}}  
\newcommand{\C}{\mathcal{C}}  

\renewcommand{\vec}[1]{\mathbf{#1}}     
\newcommand{\bv}[1]{\vec{#1}}           
\newcommand{\vecat}[2]{#1_{#2}}           
\newcommand{\prob}[2]{#1(#2)}  

\newcommand{\p}{\vec{p}}      
\newcommand{\q}{\vec{q}}      
\newcommand{\w}{\vec{w}}  
\newcommand{\B}{\bar{B}} 
\newcommand{\pat}[1]{\vecat{p}{#1}}  
\newcommand{\qat}[1]{\vecat{q}{#1}}  
\newcommand{\px}{\pat{x}}
\newcommand{\qx}{\qat{x}}
\newcommand{\pof}[1]{\prob{\p}{#1}}   
\newcommand{\qof}[1]{\prob{\q}{#1}}   

\newcommand{\point}[1]{\boldsymbol{\delta}^{#1}}  
\newcommand{\pointx}{\point{x}}

\newcommand{\ph}{\vec{\hat{p}}}    
\newcommand{\phat}[1]{\vecat{\hat{p}}{#1}}

\mathchardef\dash="2D
\newcommand{\logloss}{{log{\dash}loss}}
\newcommand{\linloss}{{lin{\dash}loss}}

\hypersetup{
    colorlinks = true,
 	linkcolor = teal,
 	citecolor = teal
}

\title{Toward a Characterization of Loss Functions for Distribution Learning}
\date{}

\author{%
    Nika Haghtalab\\
    Microsoft Research New England \\
    \texttt{nika.haghtalab@microsoft.com} \\
  \and
    Cameron Musco \\
    Microsoft Research New England \\
    \texttt{camusco@microsoft.com} \\
  \and
    Bo Waggoner \\
    Microsoft Research NYC\\
    \texttt{benjamin.waggoner@microsoft.com}
}

\begin{document}

\maketitle

\begin{abstract}
In this work we study loss functions for learning and evaluating probability distributions over large discrete domains. Unlike  classification or regression where a wide variety of loss functions are used, in the distribution learning and density estimation literature, very few losses outside the dominant \emph{log loss} are applied. We aim to understand this fact, taking an axiomatic approach to the design of loss functions for learning distributions.
We start by proposing a set of desirable criteria that any good loss function should satisfy. 
Intuitively, these criteria require that the loss function faithfully evaluates a candidate distribution, both in expectation and when estimated on a few samples. 
Interestingly, we observe that \emph{no loss function} possesses all of these criteria.
However, one can circumvent this issue by introducing a natural restriction on the set of candidate distributions.  Specifically, we require that candidates are \emph{calibrated} with respect to the target distribution, i.e., they may contain less information than the target but otherwise do not significantly distort the truth.
We show that, after restricting to this set of distributions, the log loss, along with a large variety of other losses satisfy the desired criteria. 
These results pave the way for future investigations of distribution learning that look beyond the log loss, choosing a loss function based on  application or domain need.
\end{abstract}

\thispagestyle{empty} 
\clearpage
\setcounter{page}{1}


\section{Introduction}
Estimating a probability distribution given independent samples from that distribution is a fundamental problem in machine learning and statistics \citep[e.g.][]{silverman1986density,batu2000testing,valiant2016instance,canonne2019data}. 
In machine learning applications, the distribution of interest is often over a very large but finite sample space, e.g., the set of all English sentences up to a certain length or images of a fixed size in their RGB format.

A central technique in learning these types of distributions, encompassing, e.g., log likelihood maximization, is evaluation via a \emph{loss function}.
Given a distribution $\p$ over a set of outcomes $\X$ and a sample $x \sim \p$, a loss function $\ell(\q,x)$ evaluates the performance of a candidate distribution $\q$ in predicting $x$. Generally, $\ell(\q,x)$ will be higher if $\ell$ places smaller probability on $x$. Thus, in expectation over $x \sim \p$, the loss will be lower for candidate distributions 
that closely match $\p$.

The dominant loss applied in practice is the log loss ($\ell(\q,x) = \ln(1/\qx)$), {which corresponds to log likelihood maximization}. Surprisingly, few other losses are ever considered. 
This is in sharp contrast to other areas of machine learning, including in supervised learning where different applications have necessitated the use of different losses, such as the squared loss, hinge loss, $\ell_1$ loss, etc.
However, alternative loss functions can be beneficial for density estimation on large domains, as we show with a brief motivating  example.

\paragraph{Motivating example.}
In many learning applications, one seeks to fit a complex distribution with a simple model that cannot fully capture its complexity. This includes e.g., noise tolerant or agnostic learning.
As an example, consider modeling the distribution over English words with a character trigram model. While this model, trained by minimizing log loss, fits the distribution of English words relatively well, its performance significantly degrades if the dataset includes a small fraction of foreign language words. The model is unable to fit the `tail' of the distribution  (corresponding to foreign words), however, in trying to do so it performs significantly worse on the `head' of the distribution (corresponding to common English words). This is due to the fact that minimizing log loss requires $q_x$ to not be much smaller than $p_x$ for all $x$.
A more robust loss function, such as the \emph{log log loss}, $\ell(\q,x) = \ln(\ln(1/q_x))$, emphasizes the importance of fitting the `head' and is less sensitive to  the introduction of the foreign words. See Figure \ref{fig:lang} and Appendix \ref{app:experiments} for details.
\begin{figure}[h]
	\centering
	\begin{multicols}{3}
	\scriptsize \hfill \break
		$\scriptsize
		\begin{array}{|c|c|}
		\hline
		\textbf{Samples from $\q_1$} & \textbf{Samples from $\q_2$} \\
		\hline
		\texttt{brappost} & \texttt{to}\\
		\texttt{hild} & \texttt{oneems}\\
		\texttt{me} & \texttt{the}\\
		\texttt{on} & \texttt{not} \\
		\texttt{ther} & \texttt{of}\\
		\hline
		\end{array}$
	\columnbreak
	
	\scriptsize \hfill \break
		$\scriptsize
		\begin{array}{|c|}
		\hline
	  \texttt{log loss}(\p) = 7.45 \\
	   \texttt{log log loss}(\p) = 1.91\\
	   			 \hline 
		\texttt{log loss}(\q_1) = 11.25\\
		\texttt{log log loss}(\q_1) = 2.22 \\
		\hline
		 \texttt{log loss}(\q_2) = 12.26 \\
	 \texttt{log log loss}(\q_2) = 2.18\\ 
	   \hline
		\end{array}$
	\columnbreak
	
		\includegraphics[width=.33\textwidth]{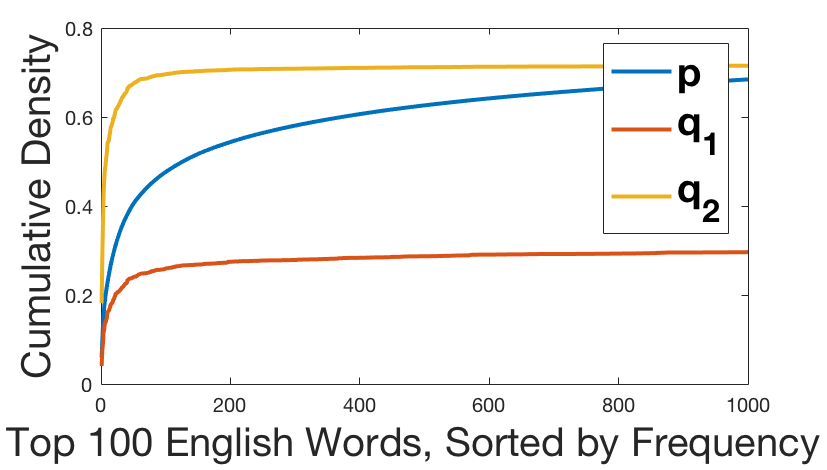}
		\end{multicols}
	\vspace{-1em}
	\caption{\small Modeling the distribution of English words, corrupted with $12\%$ French and German words with character trigrams. Distribution $\q_1$ is trained by minimizing log loss. $\q_2$ achieves worse log loss but better \emph{log log loss} and better performance at fitting the `head' of the the target $\p$, indicating that log log loss may be more appropriate in this application. See Appendix \ref{app:experiments} for more details.}\label{fig:lang}
\end{figure}

\vspace{-.5em}

\paragraph{Loss function properties.}
In this paper, we start by understanding the desirable properties of log loss and seek to identity other loss functions with such properties that can have applications in various domains.
A key characteristic of the log loss is that  it is (strictly) \emph{proper}. That is, the true underlying distribution  $\p$ (uniquely) minimizes the expected loss on samples drawn from $\p$.
Properness is essential for loss functions, as without it minimizing the expected loss leads to choosing an incorrect candidate distribution even when the target distribution is fully known.
Log loss is also \emph{local} (sometimes termed \emph{pointwise}). That is, the loss of $\q$ on sample $x$ is a function of the probability $\qx$ and not of $q_{x'}$ for $x'\neq x$.
Local losses are preferred in machine learning, where $q_x$ is often implicitly represented as the output of a likelihood function applied to $x$, but where fully computing $\q$ requires at least linear time in the size of the sample space $N$ and is infeasible for large domains, such as  learning the distribution of all English sentences up to a certain length.

It is well-known that \emph{log loss is the unique local and strictly proper loss function}~\citep{mccarthy1956measures,savage1971elicitation,gneiting2007strictly}. 
Thus, requiring strict properness and locality already restricts us to using the log loss. At the same time, these restrictive properties are not sufficient for effective distribution learning, because
\begin{itemize}[leftmargin=2.5em]
  \item A candidate distribution may be far from the target  yet have arbitrarily close to optimal loss.
        Motivated by this problem, we define \emph{strongly proper} losses that, if given a candidate far from the target,  will give an expected loss significantly worse than optimal.
  \item A candidate distribution might be far from the target, yet on a small number of samples, it may be likely to have smaller empirical loss than that of the target.
        This motivates our definition of \emph{sample-proper} losses.
 \item On a small number of samples, the empirical loss of a distribution may be far from its  expected loss, making evaluation impossible.
        This motivates our definition of \emph{concentrating} losses.
\end{itemize}

Naively, it seems we cannot satisfy all our desired criteria: our only local strictly proper loss is the log loss, which in fact fails to satisfy the concentration requirement (see Example~\ref{ex:log-loss-concentrate}).
We propose to overcome this challenge by restricting the set of candidate distributions, specifically to ones that satisfy the reasonable condition of \emph{calibration}.
We then consider the properties of loss functions on, not the set of all possible distributions, but the set of calibrated distributions.

\paragraph{Calibration and results.}
We call a candidate distribution $\q$ calibrated with respect to a target $\p$ if all elements to which $\q$ assigns probability $\alpha$ actually occur on average with probability $\alpha$ in the target distribution.\footnote{This definition is an adaptation of the standard calibration criterion applied to sequences of predictions made by a forecaster \citep{dawid1982well,foster1998asymptotic}. See discussion in Appendix \ref{app:calibrated}.}
This can also be interpreted as requiring $\q$ to be a coarsening of $\p$, i.e., a calibrated distribution may contain less information than $\p$ but does otherwise not distort information.
While for simplicity we focus on exactly calibrated distributions, in Appendix \ref{app:approx} we extend our results to a natural notion of approximate calibration.
Our main results show that the calibration constraint overcomes the impossibility of satisfying properness along with the our three desired criteria.

\begin{mainres*}[Informal summary]
  Any (local) loss $\ell(\q,x) \eqdef f\left(\frac{1}{\qx}\right)$ such that $f$ is strictly concave and monotonically increasing has the following properties subject to calibration:
  \begin{enumerate}
    \item $\ell$ is \emph{strictly proper}, i.e., the target distribution minimizes expected loss.
    \item If $f$ furthermore satisfies \emph{left-strong-concavity}, $\ell$ is \emph{strongly proper}, i.e., distributions far from the target have significantly worse loss.
    \item If $f$ furthermore grows relatively slowly, $\ell$ is \emph{sample proper} i.e., on few samples, distributions far from the target have higher empirical loss with high probability.
    \item Under these same conditions, $\ell$ \emph{concentrates} i.e., on few samples, a distribution's empirical loss is a reliable estimate of its expected loss with high probability.
  \end{enumerate}
\end{mainres*}

The above criteria are formally introduced in Section~\ref{sec:criteria}. Each criteria is parameterized and different losses satisfy them with different parameters. We illustrate a few examples in Table \ref{tab:intro} below.
We emphasize that all losses shown below achieve relatively strong bounds, only depending polylogarithmically on the domain size $N$. Thus, we view all of these loss functions as viable alternatives to the log loss, which may be useful in different applications.

\begin{table}[h]
\centering
{\small
\begin{tabular}{c|c|c|c}
 \multirow{2}{*}{$\ell(\q,x)$}  & Strong Properness  & Concentration  & Sample Properness\\
                         &$\E \ell(\q;x)-\E \ell(\p;x)$  &  sample size $m(\gamma,N)$ &  sample size  $m(\epsilon,N)$\\
  \hline
  $\ln \frac{1}{\qx}$ & $\Omega(\epsilon^{2})$ &  $\tilde O \left ( \gamma^{-2}  \ln \left (\frac{N}{\gamma} \right )^2  \right )$  &  $O \left(\epsilon^{-4}\left (\ln N \right )^2 \right )$\\
     $\left (\ln\frac{1}{\qx}\right )^p$ for $p \in (0,1]$  & $\Omega \left (\epsilon^{2} \left (\ln N \right )^{p-1}\right )$   & $\tilde O \left (\gamma^{-2}  \ln \left (\frac{N}{\gamma} \right )^{2p}  \right )$ & $O \left (\epsilon^{-4}\left (\ln N \right )^2    \right )$\\
  $\ln\ln \frac{1}{\qx} $ & $\Omega \left ( \frac{\epsilon^{2}}{\ln N} \right )$ & $\tilde O \left (\gamma^{-2}  \ln\ln \left (\frac{N}{\gamma} \right )^2    \right )$ & $O \left (\epsilon^{-4}(\ln\ln N)^2   (\ln N)^2  \right )$\\
 $\left(\ln\frac{e^2}{\qx}\right)^2$  & $\Omega(\epsilon^2)$ &  $\tilde O \left ( \gamma^{-2}  \ln \left (\frac{N}{\gamma} \right )^{4} \right )$& $ O \left (\epsilon^{-4}(\ln N)^4 \right )$
\end{tabular}
  \caption{\small Examples of loss function that demonstrate strong properness, sample properness, and concentration, when restricted to calibrated distributions. In the above, $N$ is the distributions support size, $\epsilon \eqdef \norm{\p-\q}_1$ is the $\ell_1$ distance between $\p$ and $\q$, and $\gamma$ is an approximation parameter for concentration (see Section \ref{sec:general} for details). We assume for simplicity that $\epsilon \ge 1/N$ and hide dependencies on a success probability parameter for sample properness and concentration. $\tilde O(\cdot)$ suppresses logarithmic dependence on $1/\epsilon$ and $1/\gamma$.
}\label{tab:intro}
  }
  \end{table}
\subsection{Related work}

Our work is directly inspired  by applications of distribution estimation in very high-dimensional spaces, such as language modeling \cite{languagemodeling}. However, we do not know of work in this area that takes a systematic approach to designing loss functions.

A conceptually related research problem is that of learning distributions using computationally and statistically efficient algorithms. Beyond loss function minimization, a number of general-purpose methods have been proposed for this problem, including using histograms, nearest neighbor estimators, etc. See \citep{izenman1991review} for a survey of these methods.
Much of the work in this space focuses on learning \emph{structured} or  \emph{parametric} distributions~\cite{daskalakis14,disentangling, kearns1994learnability, chan2013learning}, e.g.,  monotone distributions or mixtures of Gaussians. On the other hand, learning an unstructured discrete distribution with support size $N$ within $\ell_1$ distance $\epsilon$ requires $\poly(N,1/\epsilon)$ samples. Thus, works in this space typically focus on designing computationally efficient algorithms for optimal estimation using large sample sets~\cite{valiant2016instance}. In comparison, we focus on unstructured distributions with prohibitively large supports and characterize loss functions that only require $\mathrm{polylog}(N)$ sample complexity to estimate. We do not introduce a general algorithm for distribution learning --- as any such algorithm would require  $\Omega(N)$ samples. Rather, 
motivated by tailored algorithms used in complex domains such as natural language processing, our work characterizes loss functions that could be used by a variety of algorithms.

Outside distribution learning, loss functions (termed \emph{scoring rules}) have been studied for decades in the information elicitation literature, which seeks to incentivize experts, such as weather forecasters, to give accurate predictions \citep[e.g.][]{brier1950verification,good1952rational,savage1971elicitation,frongillo2015vector, gneiting2007strictly}.
The notion of loss function properness, for example, comes from this literature.
Recent research has made some connections between information elicitation and loss functions in machine learning; however, it has focused mostly on the classification and regression and not distribution learning \citep{agarwal2015consistent,frongillo2015vector, NarasimhanRS015,ramaswamy2016convex,ehm2016quantiles}.
Our work can be viewed as a contribution to the literature on evaluating forecasters by showing that, if the forecaster is constrained to be calibrated, then a variety of simple local loss functions become (strongly, sample) proper.


\section{Preliminaries} \label{sec:background}
We work with distributions over a finite domain $\X$ with $|\X| = N$. The set of all distributions over $\X$ is denoted by $\Delta_\X$. We denote a distribution $\vec p\in\{0,1\}^N$ over $\X$ by a vector of probabilities, where $p_x$ is the probability $\vec p$ places on $x\in \X$. For any set $B\subseteq \X$,  the total probability $\vec p$ places on $B$ is denoted by $\pof{B} \eqdef \sum_{x \in B} \px$. 
We use $X$ to denote a random variable on $\X$ whose distribution is specified in context.
We also consider point mass distributions $\pointx \in \Delta_{\X}$ where $\vecat{\pointx}{x'} = \Ind{x=x'}$.

Throughout this paper, we typically use $\p$ to denote the true (or target) distribution and $\q$ to denote a candidate or predicted distribution. For any two distributions $\p$ and $\q$, the  \emph{total variation distance} between them is defined by $\TV(\p,\q) \eqdef \sup_{B \subseteq \X} \pof{B} - \qof{B} = \frac{1}{2}\|\p-\q\|_1$, where $\|\cdot\|_1$ denotes the $\ell_1$ norm of a vector.
Together, $\ell_1$ and the total variation distance are two of the most widely used measures of distance between distributions.

To measure the quality of a candidate distribution $\q$ given samples from $\p$, machine learning typically turns to loss functions.
A \emph{loss function} is a function $\ell: \Delta_{\X} \times \X \to \R$ where $\ell(\q,x)$ is the loss assigned to candidate $\q$ on outcome $x$.
Given a target distribution $\p$, the \emph{expected loss} for candidate $\q$ is defined as $\ \ell(\q;\p) \eqdef \E_{X\sim \p} \left[ \ell(\q,X)\right] .$
A loss function is called \emph{proper} if $\ell(\p;\p) \leq \ell(\q;\p)$ for all $\p\neq \q$, and \emph{strictly proper} if the inequality is always strict\footnote{Our use of ``properness'' is inspired the literature on \emph{proper scoring rules}. It is not to be confused with ``properness'' in learning theory where the learned hypothesis must belong to a pre-determined class of hypotheses.}.
Two common examples of proper loss functions are the  \emph{log loss} function $\ell(\q,x) = \ln( \tfrac{1}{\qx})$ (with the logarithm always taken base $e$ in this paper) and the \emph{quadratic loss} $\ell(\q,x) = \frac{1}{2}\|\pointx - \q\|_2^2$. 
A loss function $\ell$ is called \emph{local} if $\ell(\q,x)$ is a function of $\qx$ alone.
For example, the log loss is local while the quadratic loss is not. 

Our main results are characterized by the {geometry} of the loss functions we consider. 
For simplicity, we will generally assume functions are differentiable, although our results can be extended.

\begin{definition}[Strongly Concave]\label{def:strongConcavity}
A function $f: [0,\infty] \to \R$ is \emph{$\beta$-strongly concave}  if for all $z,z'$ in the domain of $f$, $f(z) \leq  f(z') + \nabla f(z') \cdot (z-z') - \frac{\beta}{2}(z-z')^2.$
\end{definition}
We also consider a relaxation of strong concavity that helps us in analyzing functions that have a large curvature close to the origin but flatten out as we move farther from it.

\begin{definition}[Left-Strongly  Concave]\label{def:leftStrongConcavity}
A function $f: [0,\infty] \to \R$ is \emph{$\beta(z)$-left-strongly concave} if the function restricted to $[0,z]$ is $\beta(z)$-strongly concave, for all $z$.
\end{definition}

As discussed, a natural assumption on the set of candidate distributions is  \emph{calibration}. Formally:

\begin{definition}[Calibration] \label{def:calibrated}
Given a distribution $\q \in \Delta_{\X}$, let $B_t(\q) = \{x : \qx = t\}$. When it is clear from the context, we suppress $\q$ in the definition of $B_t$. 
We say that $\q$ is \emph{calibrated with respect to $\p$}, if $\qof{B_t(\q)} = \pof{B_t(\q)}$ for all $t\in [0,1]$.
We let $\C(\p)$ denote the set of all calibrated distributions with respect to $\p$.
\end{definition}
In other words, $\q$ is calibrated with respect to $\p$ if points assigned probability $\qx = t$ have average probability $t$ under $\p$.
In other words, $\p$ can be ``coarsened'' to $\q$ by taking subsets of points and replacing their probabilities with the subset average.
Note that the uniform distribution $\q = (\tfrac{1}{N},\dots,\tfrac{1}{N})$ is calibrated with respect to all $\p$, and that $\p$ is calibrated with respect to itself.
Also note that there are only finitely many values $t \in [0,1]$ for which $B_t$ is non-empty. We denote the set of these values by $T(\q) = \{t: B_t \neq \emptyset\}$.

We refer an interested reader to Appendix \ref{app:calibrated} for a more detailed  discussion of the notion of calibration and its connections to similar notions used in forecasting theory,~e.g. \citep{dawid1982well,foster1998asymptotic}. See Appendix \ref{app:approx} for a discussion of how our results can be extended to a natural notion of approximate calibration.


\section{Three Desirable Properties of Loss Functions} \label{sec:criteria}
In this section, we define three criteria and discuss why any desirable loss function should demonstrate them.  We use examples of loss functions, such as the log loss $\ell_\logloss(\q,x) = \ln (\frac{1}{\qx})$ and the linear loss $\ell_\linloss(\q,x) = - \qx$ to help demonstrate the existence or lack of these criteria.

\subsection{Strong Properness}
Recall that a loss function is strictly proper if all incorrect candidate distributions yield a higher expected loss value than the target distribution. 
Here, we expand this to \emph{strong} properness where this gap in expected loss grows with distance from the target distribution.
We also extend both definitions to hold over a specific domain of candidate distributions, rather than all distributions.

\begin{definition}[Calibrated Properness]
Let  $\P: \Delta_{\X} \to 2^{\Delta_{\X}}$ be a \emph{domain function}, that is, $\P(\p)\subseteq \Delta_\X$ is a restricted set of distributions. 
A loss function $\ell$ is \emph{proper over $\P$} if for all $\p \in \Delta_{\X}$, $\p \in \argmin_{\q \in \P(\p)} \ell(\q;\p) .$
A loss function is said to be \emph{strictly proper over $\P$} if the argmin is always unique. When $\P(\p) = \C(\p)$, i.e. is the set of calibrated distributions w.r.t. $\vec p$, we call such a loss function \emph{(strictly) calibrated proper}.
\end{definition}

\begin{example}\label{ex:log-loss-diff-kl}
It is well-known that $\ell_{\logloss}(\vec q, x) = \ln\left(\frac{1}{q_x}\right)$ is the \emph{unique local} proper loss function (up to scaling) over the unrestricted domain $\P(\vec p) = \Delta_\X$~\citep{bernardo1979expected}. Indeed, it is known that  the difference in expected log loss of a prediction $\vec q$ and the target distribution $\vec p$ is the KL-divergence, i.e.
\begin{align} 
\label{eq:diff=KL}
 \ell_{\logloss}(\q;\p) - \ell_{\logloss}(\p;\p) = \KL(\p,\q) \eqdef ~ \sum_x \px \ln \left(\frac{\px}{\qx}\right).
 \end{align}
Furthermore, the KL-divergence is strictly positive for $\p \neq \q$. This proves that the log loss is strictly proper over $\Delta_\X$, and as a result, is strictly calibrated proper as well. 

On the other hand, $\ell_{\linloss}(\q,x) = -q_x$ is not proper over $\Delta_\X$. This is due to that fact that the minimizer of this loss is the point mass distribution $\pointx$ for $x = \argmax_x p_x$. For example, for target distribution  $\vec p = (\frac 13, \frac 23)$, distribution $\vec q = (0, 1)$ yields a lower $\ell_{\linloss}$ than that of $\vec p$. Note, however, that such a choice of $\vec q$ is not calibrated with respect to $\vec p$. When loss minimization is constrained to the set of calibrated distributions, $\C(\vec p) = \{(\frac 13, \frac 23), (\frac 12, \frac 12)\}$, $\vec p$ minimizes the expected linear loss. Indeed, in Section~\ref{sec:mainresults}  we show more generally that the linear loss and in fact many reasonable local loss functions are calibrated proper.
\end{example}

While strict properness is an important baseline guarantee, we would like a ``stronger'' property: If $\q$ is significantly incorrect in the sense of being far from $\p$, then the expected loss of $\q$ should be significantly worse.
This motivates the following definition.

\begin{definition} [Strong Calibrated Properness]
  A loss function $\ell$ is \emph{$\beta$-strongly proper over a domain function $\P$} if for all $\p \in \Delta_{\X}$, for all $\q \in \P(\p)$, $\ell(\q;\p) - \ell(\p;\p) \geq \frac{\beta}{2} \norm{\p - \q}_1^2 .$
When $\P(\vec p) = \C(\vec p)$, we call such functions \emph{$\beta$-strongly calibrated proper} and when $\P(\p) = \Delta_{\X}$, we simply refer to them as \emph{$\beta$-strongly proper}.
\end{definition}

\begin{example} \label{ex:strongly-proper}
The log loss is $1$-strongly proper.
This is \emph{equivalent} to Pinsker's inequality, which states
that \emph{for all $\vec p$ and $\vec q$,} $\KL(\p,\q) \geq 2 \TV(\p,\q)^2$.  
Together with~\eqref{eq:diff=KL} and the fact that $\TV(\vec p, \vec q) = \frac 12 \norm{\vec p - \vec q}_1$, this shows that log loss is $1$-strongly proper (and thus also $1$-strongly calibrated proper.)
\end{example}

As we will see in Section~\ref{sec:mainresults}, strong calibrated properness relates to the notion of strong concavity (of the inverse loss function) in $\ell_1$ norm.
We refer the interested reader to Appendix \ref{app:strong-l2} for a discussion of the use of alternative norms in the definition of strong properness. In Appendix~\ref{app:strong-proper-simplex} we extend the study of normed concavity of loss functions to strong properness of a loss function over $\Delta_{\X}$ .

\subsection{Sample-properness}
So far, we have focused on the loss a candidate $\vec q$ receives in \emph{expectation over $x\sim \vec p$}. Of course, if one is attempting to learn $\p$, this expectation can generally not be computed.
We would like the notion of properness  to carry over to the setting when the loss on $\q$ is estimated using 
a small set of samples from $\vec p$. 
We say that a loss function is sample-proper if within a small number, all candidate distributions that are sufficiently far from $\vec p$ yield a loss that is larger than that of $\vec p$ on the samples.

In the remainder of this paper, let $\ph$ denote the empirical distribution corresponding to samples drawn from $\vec p$. Note that the average loss of any $\q$ on the samples can be written $\ell(\q;\ph)$. Formally:

\begin{definition}[Calibrated Sample-Properness]
  A loss function $\ell$ is \emph{$m(\epsilon,\delta,N)$-sample proper} over a function domain $\P$ if, for all $\p \in \Delta_{\X}$ and all $\q \in \P(\p)$ with $\norm{\p-\q}_1 \geq \epsilon$, with probability at least $1-\delta$ over $m(\epsilon,\delta,N)$ i.i.d. samples from $\p$, we have $\ell(\p;\ph) < \ell(\q;\ph)$.
When $\P(\vec p) = \C(\vec p)$, we call such functions \emph{calibrated $m(\epsilon,\delta,N)$-sample proper.}  
\end{definition}

\begin{example}
A folklore theorem states that $\ell_{\logloss}$ is $O\left(\frac{1}{\epsilon^2}\ln\left( \frac 1\delta \right)\right)$-sample proper over $\Delta_\X$, and as a result it is calibrated $O\left(\frac{1}{\epsilon^2}\ln\left( \frac 1\delta \right)\right)$-sample proper.

Now consider $\ell_\linloss(\q,x) = - \qx$. Since it is not a proper loss function over $\Delta_\X$, by definition it is not sample proper over $\Delta_\X$ for any  $m(\epsilon, \delta, N)$. When restricting to calibrated distributions however, as we claimed in Example~\ref{ex:log-loss-diff-kl} linear loss is calibrated proper in expectation.It is interesting to note that linear loss is not sample proper for any $m(\epsilon, \delta, N) \in o\left(N\right)$. To observe this, consider $\vec p$ where $p_1 = \frac 14{+}\frac{1}{\sqrt m}$, $p_2=\frac 14{-}\frac{1}{\sqrt m}$, and $p_x = \frac{1}{2(N/2 - 2)}$ for $x = 3, \dots, N/2$ and $p_x = 0$ for $x = N/2+1, \dots, N$.
Consider $\vec q$ where $q_1 = q_2 = \frac 14$ and $q_x = \frac{1}{2(N - 2)}$ for $x = 3, \dots, N$. 
Let $\ph$ be the empirical distribution. With a constant probability, $\hat{p}_1 \leq \frac 14 - \frac{1}{\sqrt{m}}$ and $\hat{p}_2 \geq \frac 14$. Let $\nu = \frac{1}{2(N/2 - 2)} - \frac{1}{2(N - 2)} = \Theta(\frac 1N) $. Therefore,
\begin{align*}
\ell(\vec q; \ph) - \ell(\p; \ph) 
	&=   \sum_{x=1}^N  \hat p_x (p_x - q_x) \\
	&=  \frac{\hat p_1}{\sqrt{m}} + \frac{ -\hat p_2 }{\sqrt{m}}  + \nu \sum_{x=3}^{N/2} \hat p_x -   \nu \sum_{x=N/2+1}^{N} \hat p_x    \\
		&=  \frac{1}{\sqrt{m}} \left( \frac 14 - \frac{1}{\sqrt m}\right) - \frac{1}{\sqrt{m}}  \frac 14  + \Theta\left( \frac 1N \right) \\
		&= - \frac{1}{m} + \Theta\left( \frac 1N \right) < 0,
\end{align*}
when $m \in o\left(N\right)$. Furthermore, note that $\vec q$ is calibrated w.r.t. $\vec p$ with two non-empty buckets $B_{\frac 14}(\vec q) = \{1, 2\}$ and  $B_{\frac{1}{2(N-2)}}(\vec q) = \{3, \dots, N\}$. Moreover, $\| \vec p - \vec q\|_1 = \Theta(1)$. Thus, for
$\ell_\linloss$ to be calibrated $m(\epsilon, \delta, N)$-sample proper, we must have $m(\Theta(1), \Theta(1), N)\in \Omega\left(N\right)$.
\end{example}

\subsection{Concentration}
Beyond sample properness, when 
the expected loss $\ell(\q;\p)$ is estimated from a small i.i.d. sample from $\vec p$, we would like the empirical loss to remain faithful to the true value.
For example, 
one might hope that minimizing loss on that sample will result in a distribution that has small loss on $\p$. This will hold as long as the empirical loss well approximates the true expected loss with high probability.

\begin{definition}[Calibrated Concentration] \label{def:gen}
  A loss function $\ell$ \emph{concentrates over domain function $\P$ with $m(\gamma,\delta,N)$ samples} if for all $\p \in \Delta_{\mathcal{X}}$, for all $\q \in \P(\p)$, for $m(\gamma,\delta,N)$ i.i.d. samples from $\p$, 

$  \Pr \left [ \left | \ell(\q; \ph) -\ell(\q;\p) \right | \ge \gamma \right ] \le \delta.
$
When $\P(\vec p) = \C(\vec p)$, we say that $\ell$ \emph{calibrated concentrates} with $m(\gamma,\delta,N)$ samples.\footnote{We use $\gamma$ to denote difference in loss to avoid confusion with $\epsilon$, which generally means a distance between distributions.}
\end{definition}

\begin{example}
\label{ex:log-loss-concentrate}
\label{ex:logLossDoesntGeneralize}
We can easily see that log loss does \emph{not} concentrate with $o(N)$ samples over $\Delta_\X$. 
Let $\p$ be the uniform distribution and $\q$ be uniform on $\X \setminus \{x\}$.
With high probability, $x$ is not sampled, and $\ell(\q;\ph)$ is finite.
Yet $\ell(\q;\p) = \infty$.
Note that although this example is extreme, its conclusion is robust:  one can make an arbitrarily large finite gap. As we will see, the log loss, along with many other reasonable loss will concentrate with a small number of samples over calibrated distributions.
\end{example}


\section{Main Results} \label{sec:mainresults}

Looking back at the criteria defined in Section~\ref{sec:criteria}, we are immediately faced with an impossibility result: no local loss function exists that satisfies properness, $o(N)$-sample properness, and concentration with $o(N)$ samples.
This is because log loss is the unique local loss function that satisfies the first property and as shown in Example~\ref{ex:log-loss-concentrate} it does not concentrate.
In this section, we show that a broad class of local loss functions with certain niceness properties satisfies the above three criteria over calibrated domains.
Specifically, we  consider loss functions $\ell(\vec q, x)$ that are non-increasing in $q_x$ and are  inversely concave: $\ell(\q,x) = f(\frac{1}{\qx})$ for some concave function $f$. Similarly, we say that $\ell$ is inversely strongly concave if the corresponding $f$ is strongly concave. 

\subsection{Calibrated and Strong Calibrated Properness} \label{sec:proper}
In this section, we show that any (strongly) nice loss function is (strongly) proper over the domain of calibrated distributions. More formally. 
\begin{theorem}[Strict Properness]\label{thm:strict_proper}
Suppose the local loss function $\ell$ is such that  $\ell(\vec q, x)= f( \frac {1}{q_x})$ for a concave $f$ function. Then, $\ell$ is strictly proper over the domain function $\mathcal{C}$.
\end{theorem}

\begin{theorem}[Strong Properness] \label{thm:nice-strong-proper}
Suppose the loss function $\ell$ is such that $\ell(\q,x) = f(\frac{1}{\qx})$ where $f$ is non-decreasing and is $\frac{C(x)}{x^2}$-left-strongly concave where $C(x)$ is non-increasing and non-negative for $x \ge 1$.
Then for all $\p \in \Delta_{\mathcal{X}}$ and $\q \in \mathcal{C}(\p)$, 
\[\ell(\q;\p) - \ell(\p;\p) \ge C\!\left(\frac{4N}{\norm{\p-\q}_1}\right) \cdot \frac{ \norm{\p-\q}_1^2}{128}.
\]
\end{theorem}

We begin with the proof of Theorem~\ref{thm:strict_proper}, which relies on a key property of calibration stated in Lemma \ref{lemma:calibrated-1-over}.
At a high level, this lemma shows that the average value of $1/p_x$ and $1/q_x$ is the same over  instances $x$ such that $q_x = t$, which is also equal to $1/t$.

\begin{lemma} \label{lemma:calibrated-1-over}
For any distribution $\p\in \Delta_\X$ and $\vec q\in \C(\vec p)$, and for any $t \in [0,1]$, we have $\E_{X \sim \p} \left[ \frac{1}{\pat{X}} ~\Big|~ X \in B_t \right] = \frac{1}{t}$, where $B_t = \{x : q_x = t\}$.
\end{lemma}
\begin{proof}
We have
  \begin{align*}
    \E\left[ \frac{1}{\pat{X}} ~\Big|~ X \in B_t \right]
    = \sum_{x \in B_t} \frac{\px}{\pof{B}} \frac{1}{\px} 
    = \frac{|B_t|}{\pof{B_t}} 
    = \frac{1}{t} .
  \end{align*}
\end{proof}

\begin{proof}[Proof of Theorem~\ref{thm:strict_proper}]
  Suppose $\ell(\q,x) = f(\frac{1}{\qx})$ for a strictly concave $f$.
Consider any $\q$ that is calibrated with respect to $\p$. Recall that $B_t = \{x : q_x = t\}$ and $T(\vec q) = \{t: |B_t|\neq \emptyset\}$ is a finite set.
  \begin{align*}
    \ell(\p;\p)
& =   \sum_{t\in T(\vec q)} \pof{B_t} \E\left[ f\left(\frac{1}{\pat{X}}\right) ~\Big|~ X \in B_t \right] 
 \le \sum_{t\in T(\vec q)} \pof{B_t} f\left(\E\left[ \frac{1}{\pat{X}} ~\Big|~ X \in B_t\right]\right) \\
    &=    \sum_{t\in T(\vec q)} \pof{B_t} f\left(\frac{1}{t}\right)  
=   \sum_{t\in T(\vec q)} \sum_{x \in B_t} \px f\left(\frac{1}{\qx}\right)  =   \ell(\q;\p),
  \end{align*}
where the second transition is by Jensen's inequality and the third transition is by Lemma~\ref{lemma:calibrated-1-over}.
  If $f$ is strictly concave and there exists a $B_t$ where $\q$ and $\p$ disagree, then the inequality is strict.
\end{proof}

To prove Theorem \ref{thm:nice-strong-proper} we use an analogous result to Lemma~\ref{lemma:calibrated-1-over}, whose proof we defer to Appendix~\ref{app:lemma:strong-concave-gap}.

\begin{lemma} \label{lemma:strong-concave-gap}
  Suppose $f(z)$ is $b(z)$-left-strongly concave.
  Let $B \subseteq \X$ be any set {and let $t(B) \eqdef \frac{\pof{B}}{|B|}$},\footnote{When $B = B_t(\vec q)$ for some $t\in [0,1]$, $t(B) = t$.} and suppose $\sum_{x \in B} |\px - t(B)| \geq \epsilon$.
  Let $\mu = \frac{1}{t(B)}$.
  Then
    \[ \E_{X \sim \p} \left[ f\left(\frac{1}{\pat{X}}\right)  ~\Big|~ X \in B \right] \leq f(\mu) +  \frac{b(\mu)}{32} \frac{\epsilon^2}{\pof{B}^2 t(B)^2} . \]
\end{lemma}
\begin{proof}[Proof of Theorem \ref{thm:nice-strong-proper}]
Note that a calibrated distribution $\vec q$ can be thought of as a piecewise uniform distribution with pieces $\{B_t\}_{t\in T(\vec q)}$ and $\vec q(B_t) = \vec p(B_t)$.
  Let $\epsilon_t = \sum_{x \in B_t} |\px - \qx|$, with $\sum_{t\in T(\vec q)} \epsilon_t = \epsilon = \norm{\p - \q}_1$.
  Let {$\alpha = \frac{\norm{\p-\q}_1}{4}$ and let $H = \{t\in T(\vec q): t \ge \frac{\alpha}{N}\}$} refer to indices of pieces in which the two distributions place reasonably high probability.
   We have:
  \begin{align*}
    \ell(\q;\p) - \ell(\p;\p)
    &= \sum_x \px \left[ f\left(\frac{1}{\qx}\right) - f\left(\frac{1}{\px}\right) \right] = \sum_{t\in T(\vec q)} \pof{B_t} \left[ f\left(\frac{1}{t}\right) -  \E_{X|B_t}\left[ f\left(\frac{1}{\pat{X}}\right)\right] \right]
  \end{align*}
  where $\E_{X|B_t}[\cdot]$ refers to the expectation over $X \sim \p$ conditioned on $X \in B_t$.
  Now consider any fixed component $B_t$.
  The difference inside the brackets is $f\left(\frac{1}{t}\right) -  \E_{X|B_t}\left[ f\left(\frac{1}{\pat{X}}\right)\right] .$
  Intuitively, strong concavity implies there should be a significant ``Jensen gap''.
This is formalized in Lemma \ref{lemma:strong-concave-gap} of Appendix~\ref{app:proofs-strong} that shows that if $\sum_{x \in B_j} |\px - \qx| = \epsilon_j$, then
  \begin{align}\label{eq:jenseOut2}
    f\left(\frac{1}{t}\right) -  \E_{X|B_t} \left[ f\left(\frac{1}{\pat{X}}\right) \right]
    &\geq \frac{b\left(\frac{1}{t}\right)}{32} \cdot \frac{\epsilon_t^2}{t^2 \pof{B_t}^2} .
  \end{align}
  Summing over all $t\in T(\vec q)$ and
   Applying the assumption that $b(x) \geq \frac{C(x)}{x^2}$ where $C(x)$ is nonincreasing along with the fact that $t \ge \frac{\alpha}{N}$ for $t \in H$ gives
  \begin{align}
    \ell(\q;\p) - \ell(\p;\p) &\geq \sum_{t \in T(\q)} \pof{B_t} \frac{b(\tfrac{1}{t})}{32} \frac{\epsilon_t^2}{t^2 \pof{B_t}^2} \ge \sum_{t \in H} \pof{B_t} \frac{b(\tfrac{1}{t})}{32} \frac{\epsilon_t^2}{t^2 \pof{B_t}^2} \geq  \frac{C\left (\frac{N}{\alpha} \right )}{32}  \sum_{t \in H} \frac{\epsilon_t^2}{\p(B_t)}. \label{eq:almostAlph}
  \end{align}

For $t \notin H$, since $\q(B_t) = \p(B_t) \le \frac{\alpha |B_t|}{N}$ we have $\epsilon_t \le \frac{2\alpha |B_t|}{N}$. Thus we have $\sum_{t \notin H} \epsilon_t \le \frac{2\alpha}{N} |T(\vec q)\setminus H|  \le 2 \alpha$, and so correspondingly, $\sum_{t \in H} \epsilon_t \ge \epsilon - 2\alpha$. Since the bound of  \eqref{eq:almostAlph} is increasing in each $\epsilon_t$ and decreasing in each $\p(B_t)$ we can obtain a lower bound by considering its minimum when $\sum_{t \in H} \epsilon_t = \epsilon-2\alpha$ and $\sum_{t \in H} \p(B_t) = 1$. By the convexity of $(\cdot)^2$ this minimum is obtained  at $\epsilon_t = \p(B_t) \cdot (\epsilon-2\alpha)$ .

This gives an overall bound of $\ell(\q;\p) - \ell(\p;\p)\ge \frac{C\left (\frac{N}{\alpha} \right )}{32} \cdot (\epsilon-2\alpha)^2.$
{Replacing $\alpha = \frac{\norm{\p-\q}_1}{4}$} in this bound completes theorem.
\end{proof}


\subsection{Concentration} \label{sec:general}

The (strong) properness of a loss function, as discussed in Section \ref{sec:proper}, is only concerned with loss functions in expectation. In this section,  we consider finite sample guarantees.
Recall that  $\ell$ concentrates over $\P(\p)$ (Definition \ref{def:gen}) if, with $m(\gamma,\delta,N)$ samples, the empirical loss $\ell(\vec q;\vec{\hat p})$ of a distribution $\q \in \P(\p)$ is $\gamma$-close to its true loss $\ell(\vec q;\vec p)$ with probability $1-\delta$.
Concentration can be difficult to achieve: 
by Example \ref{ex:logLossDoesntGeneralize}, even the log loss does not concentrate for any sample size $o(N)$ for general $\q\in \Delta_\X$.
However, as we show below, when $\vec q$ is \emph{calibrated}, many natural loss functions, including log loss,  indeed concentrate. All that is needed is that the loss function is inverse concave, increasing, and does not grow too quickly as $\qx \to 0$.

\begin{theorem}[Concentration] \label{thm:gen}
  Suppose $\ell$ is a local loss function with $\ell(\q,x) = f\left(\frac{1}{\qx}\right)$ for nonnegative, increasing, concave $f(z)$.
  Suppose further that $f(z) \leq c\sqrt{z}$ for all $z \ge 1$ and some constant $c$.
  Then $\ell$ concentrates over the domain function $\C$ for any  $m(\gamma,\delta,N) \le N$, such that
  \begin{align*}
  m(\gamma,\delta,N) \ge \frac{c_1 \cdot f\left(\beta\right)^2 \ln\frac{1}{\delta}}{\gamma^2},
  \end{align*}
  where $c_1$ is a fixed constant and $\beta \eqdef \frac{16 N^{8}}{\delta \cdot \min(1,\gamma^{2}/c^2)}$.
  That is, for any $\p \in \Delta_\X, \q \in \C(\p)$, drawing at least $  m(\gamma,\delta,N)$ samples guarantees $|\ell(\q;\ph) - \ell(\q;\p)| \leq \gamma$ with probability $\ge 1-\delta$.
\end{theorem}

Note that $\gamma$ bounds the absolute difference between $\ell(\q;\ph)$ and $\ell(\q;\p)$. The desired difference may depend on the relative scale of the loss function. If e.g., we take $\ell(\q,x)$ and scale to obtain $\ell'(\q,x) = \alpha\cdot \ell(\q,x)$ for some $\alpha$, the desired error $\gamma$ scales by $\alpha$, $f(\beta)$ and $c$ both scale by $\alpha$, and thus we can see that the sample complexity remains fixed.

At a high level, Theorem~\ref{thm:gen} holds because calibration helps us avoid worst-case instances (as in Example~\ref{ex:logLossDoesntGeneralize}) using a very simple fact shown in Lemma \ref{lem:calBound}: when $\q$ is calibrated, we have $\frac{q_x}{p_x} \ge \frac{1}{N}$ for all $x$. This rules out very low probability events that contribute significantly  to $\ell(\q;\p)$ but require many samples to identify.
To prove Theorem~\ref{thm:gen} we partition $\X$ into $\Omega$ containing elements of very small probability, and $\mathcal{X} \setminus \Omega$.
With high probability, no element of $\Omega$ is ever sampled from $\vec p$.
Conditioned on this, the loss is bounded (and its expectation does not change much), so a concentration result can be applied.

\begin{lemma}[Calibrated Distribution Probability Lower Bound]\label{lem:calBound}
For any $\p \in \Delta_\X$ and $\q \in \C(\p)$, for any $x \in \X$, 
$\qx  \ge \frac{\px}{N}.$
\end{lemma}
\begin{proof} Let $B = \{x': q_{x'} = \qx\}$. Then by calibration we have:
$
\qx = \frac{\q(B)}{|B|} \ge \frac{\q(B)}{N} = \frac{\p(B)}{N} \ge \frac{\px}{N}.
$
\end{proof}
Note that this bound is achieved when $\q$ is the uniform distribution and $\p$ is a point distribution. We now proceed with the proof of Theorem \ref{thm:gen}.
%
{{We prove a stronger result, Proposition \ref{prop:gen}, that only uses the lower-bound property $\qx \geq \Omega(\frac{\px}{N})$} and does not require a distribution to be calibrated.
Combining this proposition with Lemma \ref{lem:calBound} immediately gives the theorem.}

\begin{proposition} \label{prop:gen}
  Suppose $\ell$ is a local loss function with $\ell(\q,x) = f\left(\frac{1}{\qx}\right)$ for nonnegative, increasing, concave $f(z)$.
  Suppose further that $f(z) \leq c z^r$ for all $z \geq 1$, some constant $c > 0$, and some constant $r < 1$.
  Given $\p$, suppose $\q$ is \textbf{any} distribution such that $\qx \geq \frac{c_2 \px}{N}$ for all $x$ and some constant $c_2 \in (0,1]$.
  Then, drawing at least $m(\gamma, \delta, N)$ samples guarantees that $\left| \ell(\q;\phat) - \ell(\q;\p)\right| \leq \gamma$ with probability $\geq 1-\delta$ if
    \[ m(\gamma, \delta, N) \geq \frac{c_1 \cdot f(\beta)^2 \ln\frac{1}{\delta}}{\gamma^2} ~, \]
  where $c_1$ is a fixed constant and $\beta \eqdef \frac{2^{2/(1-r)}N^{3/(1-r)+2}}{c_2^{r/(1-r)} \delta \cdot \min(1,[\gamma/c]^{1/(1-r)})}$.
\end{proposition}

\begin{proof}
Fix a sample size $m \le N$.
Let $\Omega \subseteq \X$ be the set of $x$'s that occur with non-negligible probability: 
$$\Omega = \left \{x : \px \ge  \frac{c_2^{r/(1-r)} \cdot \delta \cdot \min(1,[\gamma/c]^{1/(1-r)})}{2^{2/(1-r)}N^{3/(1-r) + 1}} \right \}.$$
we have $\bv{p}(\mathcal{X}\setminus \Omega) \le N \cdot  \frac{c_2^{r/(1-r)} \delta}{4N^4} \le  \frac{\delta}{4N}$ and thus for $x_1,\ldots x_m$ drawn i.i.d. from $\p$. By a union bound, letting $\mathcal{E}$ be the event that $x_1,\ldots, x_m \in \Omega$ and using that $m \le N$:
\begin{align}\label{eq:ebound}
\Pr \left [ \mathcal{E}\right ] \ge 1 - \frac{\delta}{4}.
\end{align}
We will condition on $\mathcal{E}$ going forward.
First note that for $x \in \Omega$, we can bound $\ell(\q,x)$ using Lemma \ref{lem:calBound}. Specifically, since $\qx \geq \frac{c_2 \px}{N}$ and $f$ is nondecreasing, we have:
\begin{align*}
\ell(\q,x) = f \left ( \frac{1}{q_x} \right ) \le f\left(\frac{2^{2/(1-r)}N^{3/(1-r)+2}}{c_2^{r/(1-r)} \cdot \delta \cdot \min(1,[\gamma/c]^{1/(1-r)})}\right).
\end{align*}
Denote {
$\beta \eqdef \frac{2^{2/(1-r)}N^{3/(1-r)+2}}{c_2^{r/(1-r)} \cdot \delta \cdot \min(1,[\gamma/c]^{1/(1-r)})} .$
}
Letting $z_i$ be the random variable:
\begin{align*}
z_i = \frac{1}{m} \left ( \ell(\bv{q},x_i) - \E_{x \sim \bv{p}} [\ell(\bv{ q}, x) | x \in \Omega ]\right ),
\end{align*}
we have for $x_i \in \Omega$, {$|z_i| \le \frac{f(\beta)}{m}$} (where we use that $\ell(\q,x)$ is nonnegative by assumption.) So {$\E[z_i^2 \mid x_i \in \Omega] \le f(\beta)^2/m^2$.}
Then by a standard Bernstein inequality:
{
\begin{align}
\Pr \left [ \left | \frac{1}{m} \sum_{j=1}^m \ell(\bv{q}, x_j) - \E_{x \sim \bv{p}} [\ell(\bv{ q}, x) \Large| x \in \Omega] \right | \ge \frac{\gamma}{2} \mid \mathcal{E} \right ] \le \exp\left(-\frac{\gamma^2/8}{f(\beta)^2/m+f(\beta)/m \cdot \gamma/3}\right) \le \frac{\delta}{2}\label{eq:bernsteinNow}
\end{align}
}
where the second inequality  follows if we have $m \geq \frac{c_1 f(\beta)^2 \log(1/\delta)}{\gamma^2}$ for sufficiently large $c_1$. 
By a union bound, from \eqref{eq:ebound} and \eqref{eq:bernsteinNow} we have:
\begin{align*}
\Pr \left [ \left | \frac{1}{m} \sum_{j=1}^m \ell(\bv{q}, x_j) - \E_{x \sim \bv{p}} [\ell(\bv{ q}, x) | x \in \Omega] \right | \ge \frac{\gamma}{2} \right ] \le \delta.
\end{align*}

It remains to show that the conditional expectation $\E_{x \sim \bv{p}} [\ell(\bv{ q}, x) | x \in \Omega]$ is very close to $\ell(\q;\p) = \E_{x \sim \bv{p}} [\ell(\bv{ q}, x)]$, which will give us the lemma. Intuitively, by conditioning on $x \in \Omega$ we are only  removing very low probability events, which do not have a big effect on the loss. Specifically, we need to show that:
\begin{align}
\left |\E_{x \sim \bv{p}} [\ell(\bv{ q}, x) | x \in \Omega] - \ell(\q;\p)\right | \le \frac{\gamma}{2}\label{eq:alphaBound}
\end{align}

Since $\bv p(\mathcal{X} \setminus \Omega) \le N \cdot \frac{c_2^{r/(1-r)} \delta \cdot \min(1,\gamma/c)}{4N^4}  \le \frac{c_2^{r/(1-r)} \gamma}{4N^3} \leq \frac{c_2^r \gamma}{4N^3}$, using that $f$ is nondecreasing, $f(z) \le c z^r$ for some $c$ and $r<1$, and $\qx \geq \frac{c_2 \px}{N}$:
\begin{align*}
\E_{x \sim \bv{p}} \left [\ell(\bv{ q}, x) \mid x \in \Omega \right ] &= 
\sum_{x \in \Omega} \frac{p_x}{\bv{p}(\Omega)} \cdot \ell(\bv{q},x)\\
&\le \frac{1}{1-\frac{c_2^r \min(1,\gamma/c)}{4N^3}}\cdot \sum_{x \in \Omega} {p}_x \cdot \ell(\bv{q},x)\\
& \le \left (1+\frac{c_2^r \min(1,\gamma/c)}{2N^3}\right ) \cdot \sum_{x \in \X} {p}_x \cdot \ell(\bv{q},x)\\
&\le \ell(\q;\p) + \frac{c_2^r \min(1,\gamma/c)}{2N^3} \cdot \sum_{x \in \X} p_x \cdot f \left (\frac{N}{c_2 p_x}\right )\\
&\le \ell(\q;\p) + \frac{c_2^r \min(1,\gamma/c)}{2N^3} \cdot c \cdot \frac{N^r}{c_2^r} \sum_{x\in \X} p_x^{1-r}
\\&= \ell(\q;\p) + \frac{\min(1,\gamma/c)}{2N^3} \cdot c \cdot N^{2r} \\
& \le \ell(\q;\p) +\frac{\gamma}{2}.
\end{align*}
This gives us one side of \eqref{eq:alphaBound}. On the other side we have:
\begin{align}\label{seventeen}
\E_{x \sim \bv{p}} \left [\ell(\bv{ q}, x) \mid x \in \Omega \right ] &= 
\sum_{x \in \Omega} \frac{{p}_x}{\p(\Omega)} \cdot\ell(\bv{q},x)\nonumber\\
&\ge  \sum_{x \in \Omega}{p}_x\cdot \ell(\bv{q},x)\\
& =  \ell(\q;\p)  - \sum_{x \notin \Omega} {p}_x \cdot \ell(\bv{q},x).
\end{align}
Again using that $f(z) \le c z^r$ for $r < 1$, that $\qx \geq \frac{c_2 \px}{N}$, and that for $x \notin \Omega$ we have $p_x \le  \frac{c_2^{r/(1-r)} \delta \cdot\min(1, [\gamma/c]^{1/(1-r)}) }{2^{2/(1-r)}N^{3/(1-r)+1}}$:
\begin{align*}
\sum_{x \notin \Omega} {p}_x \cdot \ell(\bv{q},x) = \sum_{x \notin \Omega}{p}_x \cdot f \left (\frac{1}{{q}_x} \right ) &\le c \sum_{x \notin \Omega}{p}_x^{1-r} \cdot \frac{N^r}{c_2^r} \leq  c\cdot N^{r+1} \cdot \frac{\gamma/c}{4N^{3}} \le \frac{\gamma}{4 N}.
\end{align*}
Combined with \eqref{seventeen} this yields the other side of \eqref{eq:alphaBound}, completing the bound and the proof.
\end{proof}

\begin{proof}[Proof of Theorem~\ref{thm:gen}]
By Lemma~\ref{lem:calBound}, for a calibrated distribution $\q$, $q_x \ge \frac{p_x}{N}$ for all $x\in \X$. Together with the assumption that $f(z) \le c \sqrt{z}$, we directly apply Proposition~\ref{prop:gen} to give the theorem.
\end{proof}


\subsection{Sample Properness}
Lastly, we turn our attention to calibrated sample properness. Recall that a loss function is sample proper if all candidate distributions that are sufficiently far from $\vec p$ have a loss that is larger $\p$ on the  empirical distribution $\ph$ corresponding to a small  number of samples from $\vec p$.
It is not hard to see that sample properness of a loss function is a direct consequence of its concentration and strong properness.  For any candidate distribution $\vec q$ for which  $\| \vec q - \vec p\|_1$ is large, strong properness (Theorem~\ref{thm:nice-strong-proper}) implies that $\ell(\vec q; \vec p)$ is significantly larger than  $\ell(\vec p; \vec p)$. Furthermore, concentration (Theorem~\ref{thm:gen}) implies that with high probability  $\ell(\vec q; \vec p)\approx \ell(\vec q; \ph)$ and $\ell(\p; \p)\approx \ell(\p; \ph)$. Therefore, with high probability, $ \ell(\q; \ph) >  \ell(\p; \ph)$.
Formally in Appendix \ref{app:sampleProp} we prove:

\begin{theorem}[Sample properness] \label{thm:fs-diff}
 Suppose $\ell$ is a local loss function with 
 $\ell(\q,x) = f(\frac{1}{\qx})$ for nonnegative, increasing, concave $f(z)$. Suppose further that $f(z) \le c\sqrt{z}$ for all $z \ge 1$ and some constant $c$ and that $f$ is $\frac{C(x)}{x^2}$-left-strongly concave for  where $C(x)$ is nonincreasing and nonnegative for $x \ge 1$. Then for all $\p \in \Delta_\X$ and $\q \in \C(\p)$, if $\ph$ is the empirical distribution constructed from $m$ independent samples of $\p$ with $m \le N$ and
 \begin{align*}
 m \ge \frac{c_1 \cdot f(\beta)^2 \ln \frac{1}{\delta}}{\left (C \left (\frac{4N}{\norm{\p-\q}_1} \right ) \norm{\p-\q}^2\right )^2},
 \end{align*}
 where $c_1$ is constant and  \small$\beta \eqdef \frac{288 N^{8}}{\delta \cdot \min \left (1,\left [C\!\left(\frac{4N}{\norm{\p-\q}_1}\right) \frac{ \norm{\p-\q}_1^2}{128c} \right]^2\right )}$\normalsize, then
  $\ell(\q;\ph) > \ell(\p;\ph)$ with prob. $\ge 1-\delta$.
\end{theorem}

\subsection{Application of the Main Results to Loss Functions}

We now instantiate Theorems \ref{thm:nice-strong-proper}, \ref{thm:gen}, and \ref{thm:fs-diff} for one example of a natural loss function $\ell(\vec q, x) = \ln\ln(\frac 1 \qx)$. Refer to Table~\ref{tab:intro} for other loss functions and see Appendix \ref{app:instantiate} for details on its derivation.

First, note that $\ln\ln(z)$ is $C(z)/z^2$-left-strongly concave for $C(z) = \frac{(1+ \ln(z))}{\ln(z)^2}$.\footnote{In Appendix~\ref{app:instantiate}, we show that function $f$ is $b(z)$-left-strongly concave if for all $z$, $f''(z) \le - b(z)$.}
Moreover, $C(z)$ is non-increasing and non-negative for $z\geq 1$ and $\ln\ln(z) \leq \sqrt{z}$. Using these, for any $\vec p$ and $\vec q\in \C(\vec p)$ such that $\norm{\p - \q}_1\geq \epsilon$ we have
\begin{itemize}
\item By Theorem~\ref{thm:nice-strong-proper}, $\ell(\q;\p) - \ell(\p; \p) \geq \Omega(\frac{\epsilon^2}{\ln(N/\epsilon)})$.
\item By Theorem~\ref{thm:gen},  an empirical distribution $\hat \p$ of $\tilde{O}\left( \gamma^{-2} \ln\ln(N)^2 \ln(1/\delta)  \right)$ i.i.d samples from $\p$ is sufficient such that $| \ell(\q; \hat \p) - \ell(\q;\p)|\leq \gamma$ with probability $1-\delta$.
\item  By Theorem~\ref{thm:fs-diff}, an empirical distribution $\hat \p$ of $\tilde{O}\left(\epsilon^{-4} \ln\ln(N\ln(N))^2 \ln(1/\delta) \ln (N) \right)$ i.i.d samples from $\p$ is sufficient such that $\ell(\q; \hat\p) > \ell(\p;\hat \p)$ with probability $1-\delta$.
\end{itemize}

\section{Discussion}

In this work, we characterized loss functions that meet three desirable properties: properness in expectation, concentration, and sample properness. We demonstrated that no local loss function meets all of these properties over the domain of all candidate distributions.
But, if one enforces the criterion of \emph{calibration} (or approximate calibration as discussed in Appendix \ref{app:approx}), then many simple loss functions have good properties for evaluating learned distributions over large discrete domains. 
We hope that our work provides a starting point for several future research directions.

One natural question is to understand how to select a loss function based on the application domain.   
Our example for language modeling, from the introduction, motivates the idea that log loss is not the best choice always. Understanding this more formally, for example in the 
framework of robust distribution learning, could provide a systematic approach for selecting loss functions based on the needs of the domain.
Our work also leaves open the question of designing compuationally and statistically efficient learning algorithms for different loss functions  under the constraint that the candidate $\q$ is (approximately) calibrated. 
One challenge in designing computationally efficient algorithms is that the space of calibrated distributions is not convex. We present some advances towards dealing with this challenge in Appendix~\ref{app:approx} by providing an efficient procedure for `projecting' a non-calibrated distribution on the space of approximately calibrated distribution. It remains to be seen if iteratively applying this procedure could be useful in designing an efficient algorithm for minimizing the loss on calibrated distributions.

\subsection*{Acknowledgements}
We thank Adam Kalai for significant involvement in early stages of this project and for suggesting the idea of exploring alternatives to the log loss under calibration restrictions. We also thank Gautam Kamath for helpful discussions.

\clearpage

\bibliographystyle{plain}
\bibliography{references}

\clearpage
\appendix


\section{Additional Proofs for Strongly Proper Losses}
\label{app:proofs-strong}

\subsection{Proof of Lemma~\ref{lemma:strong-concave-gap}}
\label{app:lemma:strong-concave-gap}

\begin{proof}
  We draw $X \sim \p$ conditioned on $X \in B$.
  Let $S = \{x \in B : \px > t(B)\}$.
  We upper-bound $f(\frac{1}{\pat{X}})$ for each realization of $X$.
  If $\pat{X} \leq t(B)$, then we simply use concavity.
  Otherwise, if $X \in S$, we use $b(z)$-left-strong-concavity.
  {Furthermore, note that by Lemma~\ref{lemma:calibrated-1-over},  $\E_{X\sim \p|B}\left[ \frac{1}{p_X} \right] = \mu$. We have:}  
  \begin{align*}
    {\E_{X\sim \vec p|B}}\left[ f\left(\frac{1}{\pat{X}}\right) \right] &\leq \E \left[f(\mu) + df(\mu) \cdot \left(\frac{1}{\pat{X}}-\mu\right) - \Ind{X \in S} \frac{b(\mu)}{2} \left(\frac{1}{\pat{X}}-\mu\right)^2  \right]  \\
            &=    f(\mu) - \frac{b(\mu)}{2} \frac{1}{\pof{B}}\sum_{x \in S} \px \left(\frac{1}{\px} - \mu\right)^2,
  \end{align*}

  Note the $\frac{1}{\pof{B}}$ term arises from conditioning on $X \in B$.
  We now lower-bound the sum, using the constraint that $\sum_{x \in B} |\px - t(B)| \geq \epsilon$, which implies that $\sum_{x \in S} \px - t(B) \geq \frac{\epsilon}{2}$.
  \begin{align*}
    \sum_{x \in S} \px \left(\frac{1}{t(B)} - \frac{1}{\px}\right)^2
    &=    \frac{\pof{S}}{t(B)^2} - \frac{2|S|}{t(B)} + \sum_{x \in S} \frac{1}{\px} .
  \end{align*}
  Fixing $\pof{S}$ and $|S|$, we get by convexity that this is minimized by $\px$ constant on $S$, therefore equal to $t(B) + \frac{\epsilon}{2|S|}$.
  So we have
  \begin{align*}
    |S| \left(t(B) + \frac{\epsilon}{2|S|}\right) \left(\frac{1}{t(B)} - \frac{1}{t(B) + \frac{\epsilon}{2|S|}} \right)^2
    &=   \left(|S|t(B) + \frac{\epsilon}{2}\right) \left(\frac{\epsilon}{2|S|\left(t(B)^2 + \frac{\epsilon t(B)}{2|S|}\right)} \right)^2  \\
    &=   \frac{|S| t(B) \epsilon^2 + \frac{\epsilon^3}{2}}{4|S|^2 t(B)^2\left(t(B) + \frac{\epsilon}{2|S|}\right)^2} .
  \end{align*}
  We consider the two cases for the larger term in the denominator.
  In the case $\frac{\epsilon}{2|S|} > t(B)$, we get
  \begin{align*}
    &\geq \frac{|S| t(B) \epsilon^2 + \frac{\epsilon^3}{2}} {4 |S|^2 t(B)^2 \left(\frac{\epsilon}{|S|}\right)^2}  \\
    &\geq \frac{|S| t(B) + \frac{\epsilon}{2}} {4 t(B)^2}  \\
    &\geq \frac{\epsilon}{4t(B)^2}  \\
    &\geq \frac{\epsilon^2}{4 \pof{B} t(B)^2}
  \end{align*}
  where the last line follows because we must have $\epsilon \leq \pof{B}$ from the definition of $\epsilon$.
  In the remaining case, we get
  \begin{align*}
    &\geq \frac{|S| t(B) \epsilon^2 + \frac{\epsilon^3}{2}} {4 |S|^2 t(B)^2 \left(2t(B)\right)^2}  \\
    &\geq \frac{\epsilon^2}{16 |S| t(B)^3}  \\
    &\geq \frac{\epsilon^2}{16 |B| t(B)^3}  \\
    &=    \frac{\epsilon^2}{16 \pof{B} t(B)^2} .
  \end{align*}
  
\end{proof}

\section{Additional Proofs for Sample Proper Losses}\label{app:sampleProp}
\subsection{Proof of Theorem~\ref{thm:fs-diff}}

In the statement of Theorem \ref{thm:fs-diff} we require that $\ell(\q,x) = f \left (\frac{1}{q_x} \right )$ for $f$ that is nonnegative, increasing, and $\frac{C(x)}{x^2}$-left-strongly concave. Further we require that $C(x)$ is non-decreasing and non-negative for $x \ge 1$. Directly applying Theorem \ref{thm:nice-strong-proper} we thus have:
\begin{align}\label{layoutTheGamma}
\ell(\q;\p) - \ell(\p;\p) \ge C\!\left(\frac{4N}{\norm{\p-\q}_1}\right) \cdot \frac{ \norm{\p-\q}_1^2}{128}.
\end{align}
Let $\gamma \eqdef C\!\left(\frac{4N}{\norm{\p-\q}_1}\right) \cdot \frac{ \norm{\p-\q}_1^2}{128}$.
Additionally, since $f(x) \le c \sqrt{z}$ for $z \ge 1$ and since $\q, \p \in \mathcal{C}(\p)$, applying Theorem \ref{thm:gen} with error parameter $\gamma/3$ and failure parameter $\delta/2$, we have for $\beta \eqdef \frac{288 N^{8}}{\delta \cdot \min(1,\gamma^{2}/c^2)}$, if $m \ge \frac{c_1 f(\beta)^2 \lg \frac{2}{\delta}}{(\gamma/3)^2}$ for large enough constant $c_1$ then the following hold, each with probability  $\ge 1-\delta/2$:
\begin{align*}
|\ell(\q;\ph) - \ell(\q;\p)| \leq \frac{\gamma}{3} \text{ and } |\ell(\p;\ph) - \ell(\p;\p)| \leq \frac{\gamma}{3}.
\end{align*}
By a union bound, with probability $\ge 1-\delta$ both bounds hold simultaneously and by \eqref{layoutTheGamma} we have:
\begin{align*}
\ell(\q;\ph) - \ell(\p; \ph) \ge \ell(\q;\p) - \ell(\p;\p) - \frac{2\gamma}{3} \ge \gamma - \frac{2\gamma}{3} > 0,
\end{align*}
which completes the theorem. Plugging the value of $\gamma$ in we see that the bound holds for
\begin{align*}
m \ge \frac{c_1 f(\beta)^2 \ln \frac{1}{\delta}}{\left (C\!\left(\frac{4N}{\norm{\p-\q}_1}\right) \cdot \frac{ \norm{\p-\q}_1^2}{128}/3\right )^2} = \frac{c_1' f(\beta)^2 \ln \frac{1}{\delta}}{\left (C\!\left(\frac{4N}{\norm{\p-\q}_1}\right) \cdot \norm{\p-\q}_1^2\right )^2} 
\end{align*} 
for large enough constant $c_1'$. Additionally, we see that:
\begin{align*}
\beta = \frac{288 N^{8}}{\delta \cdot \min \left (1,\left [C\!\left(\frac{4N}{\norm{\p-\q}_1}\right) \cdot \frac{ \norm{\p-\q}_1^2}{128c} \right]^2 \right )}.
\end{align*}

\section{Instantiation of Theorems~\ref{thm:nice-strong-proper}, \ref{thm:gen}, and \ref{thm:fs-diff}}
\label{app:instantiate}

Let us start with two observations regarding loss functions, characterizing inverse concave loss functions and inverse left-concave functions.

\begin{observation} \label{obs:decrease-convex-nice}
Let $\ell(\vec q, x) = f\left( \frac {1}{q_x}\right)$ be such that $\ell$ is  nonnegative, twice differentiable, decreasing, and convex. Then, $f(x)$ is concave.
\end{observation}
\begin{proof}
For ease of exposition, {let $\ell(z) = f(\frac{1}{z})$.}
  \begin{align*}
    \frac{df}{dy}     &= \frac{d\ell(\frac{1}{y})}{dz} \left(\frac{-1}{y^2}\right)  \\
    \frac{d^2f}{dz^2} &= \frac{d^2\ell(\frac{1}{y})}{dz^2} \left(\frac{-1}{y^2}\right) + \frac{d\ell(\frac{1}{y})}{dz} \left(\frac{2}{y^3}\right) .
  \end{align*}
  Decreasing and convex gives a negative derivative and positive second derivative.
  Given that $y > 0$, we obtain a negative second derivative, hence concavity.
\end{proof}

\begin{observation}\label{eq:derivToConvex}
Consider a nonincreasing function $b(z)$. A function $f$ is $b(z)$-left-strongly concave if for all $z$, $f''(z) \le - b(z)$.
\end{observation}

\begin{proof}
We need to show that $f$ restricted to $[0,z]$ is $b(z)$-strongly concave.  Consider $z_1 \ge z_2$.
Since $b(z)$ is non-increasing we have for $t \in [z_2,z_1]$:
\begin{align*}
f'(t) = f'(z_2) + \int_{z_2}^{t} f''(s) ds \le f'(z_2) - b(z) \cdot (t-z_2).
\end{align*}
We thus have:
\begin{align*}
f(z_1) - f(z_2) = \int_{z_2}^{z_1} f'(t) dt &\le \int_{z_2}^{z_1} [f'(z_2) - b(z) (t-z_2)]dt\\
&\le f'(z_2) \cdot (z_1-z_2) - b(z) \cdot \frac{(z_1-z_2)^2}{2}.
\end{align*}
Rearranging gives:
\begin{align*}
D_{-f}(z_1,z_2) \eqdef f(z_2) + f'(z_2) \cdot (z_1-z_2) - f(z_1) \ge \frac{b(z)}{2} \cdot(z_1-z_2)^2.
\end{align*}
For $z_1 \le z_2$, analogously for $t \in [z_1,z_2]$ we have:
\begin{align*}
f'(t) = f'(z_2) - \int_{t}^{z_2} f''(s) ds \ge f'(z_2) - b(z) \cdot (t-z_2)
\end{align*}
and so 
\begin{align*}
f(z_1) - f(z_2) = -\int_{z_2}^{z_1} f'(t) dt &\le \int_{z_2}^{z_1} [f'(z_2) - b(z) (t-z_2)]dt\\
&\le f'(z_2) \cdot (z_1-z_2) - b(z) \cdot \frac{(z_1-z_2)^2}{2}.
\end{align*}
Rearranging gives again gives:
\begin{align*}
 f(z_2) + f'(z_2) \cdot (z_1-z_2) - f(z_1) \ge \frac{b(z)}{2} \cdot(z_1-z_2)^2,
\end{align*}
completing the lemma.
\end{proof}

\subsection{Deriving Table~\ref{tab:intro}}
For $\ell(\vec q, x) = (\ln(1/q_x))^p$ for a constant $p\in(0,1]$.
By Observation~\ref{eq:derivToConvex}, we have that $(\ln(z))^p$ is $C(z)/z^2$-left-strongly concave for
\[
C(z) = p \ln(z)^{p-1} + p(1-p)\ln(z)^{p-2} \in \Theta\left( \ln(z)^{p-1} \right).
\]
Moreover, $C(z)$ is non-increasing and non-negative for $z\geq 1$ and $\ln(z)^{p-1} \leq \sqrt{z}$. Using these, for any $\vec p$ and $\vec q\in \C(\vec p)$ such that $\norm{\p - \q}_1\geq \epsilon$ we have
\begin{itemize}
\item By Theorem~\ref{thm:nice-strong-proper}, $\ell(\q;\p) - \ell(\p; \p) = \Omega\left( \epsilon^2 \ln(N/\epsilon)^{p-1} \right)$.
\item By Theorem~\ref{thm:gen},  an empirical distribution $\hat \p$ of ${O}\left( \gamma^{-2} \ln(1/\delta) \ln(N/\delta\gamma)^{2p}  \right)$ i.i.d samples from $\p$ is sufficient such that $| \ell(\q; \hat \p) - \ell(\q;\p)|\leq \gamma$ with probability $1-\delta$.
\item  By Theorem~\ref{thm:fs-diff}, an empirical distribution $\hat \p$ of 
\[
{O}\left(\frac{1}{\epsilon^4} \ln \left (\frac 1 \delta \right ) \ln\left( \frac{N}{\delta \epsilon^2 \ln(N/\epsilon)^p}\right)^{2p}  \ln(N/\epsilon)^{-2p+2} \right) \in{O}\left( \frac{1}{\epsilon^4} \ln \left (\frac 1 \delta \right ) \ln\left( \frac{N}{\delta \epsilon}\right)^2  \right)
\]
i.i.d samples from $\p$ is sufficient such that $\ell(\q; \hat\p) > \ell(\p;\hat \p)$ with probability $1-\delta$.
\end{itemize}

{For $\ell(\vec q, x) = \ln(e^2/q_x)^2$.
By Observation~\ref{eq:derivToConvex}, we have that $\ln(e^2\cdot z)^2$ is $\frac{2 + 2\ln(z)}{z^2}$-left-strongly concave.
Since Theorem \ref{thm:nice-strong-proper} requires that $C(z)$ is nonincreasing we cannot set $C(z) = 2 + 2\ln(z)$ as might be expected. Instead we set $C(z) = 2$. Additionally, using that $\ln(e^2\cdot z)^2 \leq \sqrt{z}$, for any $\vec p$ and $\vec q\in \C(\vec p)$ such that $\norm{\p - \q}_1\geq \epsilon$ we have

\begin{itemize}
\item By Theorem~\ref{thm:nice-strong-proper}, $\ell(\q;\p) - \ell(\p; \p) = \Omega\left( \epsilon^2 \right)$.
\item By Theorem~\ref{thm:gen},  an empirical distribution $\hat \p$ of $O\left( \gamma^{-2} \ln(1/\delta) \ln(N/\delta\gamma)^4  \right)$ i.i.d samples from $\p$ is sufficient such that $| \ell(\q; \hat \p) - \ell(\q;\p)|\leq \gamma$ with probability $1-\delta$.
\item  By Theorem~\ref{thm:fs-diff}, an empirical distribution $\hat \p$ of 
\[
{O}\left(\frac{1}{\epsilon^4} \ln \left (\frac 1 \delta \right ) \ln\left( \frac{N}{\delta \epsilon^2 \ln(N/\epsilon)}\right)^4 \right) \in O \left (\frac{1}{\epsilon^4} \ln \left (\frac 1 \delta \right )  \ln \left (\frac{N}{\delta \epsilon} \right )^4 \right )
\] 
i.i.d samples from $\p$ is sufficient such that $\ell(\q; \hat\p) > \ell(\p;\hat \p)$ with probability $1-\delta$.
\end{itemize}

\subsection{Other Loss Functions}

We also instantiate Theorem~\ref{thm:nice-strong-proper} for a few natural loss functions that do not obtain strong finite sample bounds (Theorems \ref{thm:gen}, and \ref{thm:fs-diff}).

For the linear loss $\ell_\linloss(\q,x) = - \qx$, we have by Observation \ref{eq:derivToConvex} that $-\frac{1}{z}$ is $\frac{2}{z^3}$-left-strongly-concave. Thus setting $C(z) = 1/z$, by Theorem \ref{thm:nice-strong-proper} for any $\p$ and $\q \in \mathcal{C}(\p)$ with $\norm{\p-\q}_1 \ge \epsilon$: 
$$\ell_\linloss(\q;\p)  - \ell_\linloss(\p;\p) = \Omega \left ( \frac{\epsilon}{N} \cdot \epsilon^2 \right ) = \Omega \left ( \frac{\epsilon^3}{N} \right ).$$
We can improve the dependence on $N$ and $\epsilon$ by considering e.g., $\ell(\q,x) = -\sqrt{\qx}$. In this case we have that $-1/\sqrt{z}$ is $\frac{3}{4 z^{5/2}}$-left-strongly-concave. Thus setting $C(z) = \frac{3}{4\sqrt{z}}$, by Theorem \ref{thm:nice-strong-proper} we have:
$$\ell(\q;\p)  - \ell(\p;\p) = \Omega \left ( \sqrt{\frac{\epsilon}{N}} \cdot \epsilon^2 \right ) = \Omega \left (\frac{\epsilon^{2.5}}{\sqrt{N}} \right ).$$


\section{Approximate Calibration} \label{app:approx}

In this section we show that our results are robust to a notion of approximate calibration and that we can construct distributions that satisfy approximate calibration using a small number of samples.

\begin{definition}[Approximate Calibration] \label{def:calibrated2}
For $\q \in \Delta_{\mathcal{X}}$, for any $t \in [0,1]$, let $B_t = \{x: q_t = t\}$.
$\q$ is \emph{$(\alpha_1,\alpha_2)$-approximately calibrated with respect to $\p$} if there is some subset $T \subseteq [0,1]$ such that $\qof{B_t} \in (1 \pm \alpha_1)\pof{B_t}$ for all $t \notin T$, $\qof{B_t} \ge (1-\alpha_1) \pof{B_t}$ for all $t \in T$, and {$\qof{\cup_{t\in T} B_t} \le \alpha_2$.}
Let $\C(\p,\alpha_1,\alpha_2)$ denote the set of all $(\alpha_1,\alpha_2)$-approximately calibrated distributions w.r.t. $\p$.
\end{definition}
Intuitively, $\q \in \mathcal{C}(\p, \alpha_1,\alpha_2)$ is calibrated up to $(1\pm \alpha_1)$ multiplicative error on any bucket $B_t$ where $\q$ and hence $\p$ place reasonably large mass. There is some set of buckets (corresponding to $t  \in T$) where $\q$ may significantly overestimate the probability assigned by  $\p$, however, the total mass placed on these buckets will still be small -- at most $\alpha_2$.

\subsection{Efficiently Constructing Approximately Calibrated Distributions}

We now demonstrate that, given a candidate distribution $\q$ and sample access to $\p$, it is possible to efficiently construct $\q' \in \mathcal{C}(\p,\alpha_1,\alpha_2)$. Further, if $\q \in \mathcal{C}(\p,\alpha_1,\alpha_2)$ we will have $\norm{\q-\q'}_1 \le O(\alpha_1 + \alpha_2)$. In this way, if $\q$ is approximately calibrated, we can certify at least that it is close to another approximately calibrated distribution. Of $\q$ is not approximately calibrated, we return a distribution that is approximately calibrated, which of course, may be far from $\q$.

\begin{theorem}\label{thm:const}
Given any $\q \in \Delta_{\X}$, sample access to $\p \in \Delta_\mathcal{X}$, and parameters $\alpha_1,\alpha_2,\delta \in (0,1]$ there is an algorithm that takes  $O\left (\frac{\log\left (\frac{N}{\alpha_1}\right )^2 \cdot \log\left (\frac{\log N}{\delta \alpha_1}\right)}{\alpha_1^4 \cdot \alpha_2^2} \right )$ samples from $\p$ and returns, with probability $\ge 1-\delta$, $\q' \in \C(\p,\alpha_1,\alpha_2)$. Further, if $\q \in \C(\p,\alpha_1,\alpha_2)$ then $\norm{\q-\q'}_1 \le O(\alpha_1+\alpha_2)$.
\end{theorem}
The main idea of the algorithm achieving Theorem \ref{thm:const} is to round $\q$'s probabilities into buckets of multiplicative width $(1\pm\alpha_1)$. We can then efficiently approximate the total probability mass in each bucket, excluding those that may have very small mass. On these buckets, we may over approximate the true mass, and thus they  are included in the set $T$ in Definition \ref{def:calibrated2}. 

We start with a simple lemma that shows, using a standard concentration bound, how well we can approximate the probability of any event under any distribution.
\begin{lemma}\label{lem:probEst} For any $\p \in \Delta_{\mathcal{X}}$ and $B \subseteq \mathcal{X}$, given $m$ independent samples $x_1,\ldots x_m \sim \p$, there is some fixed constant $c$ such that, for any $\epsilon,\delta \in (0,1]$, if $m \ge \frac{3 \ln(2/\delta)}{\epsilon^2}$, then with probability  $\ge 1-\delta$:
\begin{align*}
\left | \p(B) - \frac{\left | \{x_i : x_i \in B \} \right  |}{m} \right |\le \epsilon.
\end{align*}
\end{lemma} 
\begin{proof}
$\E \left | \{x_i : x_i \in B \} \right  | = m \cdot \p(B)$. By a standard Chernoff bound:
\begin{align*}
\Pr \left [ \left | \left | \{x_i : x_i \in B \} \right  |- m \cdot \p(B) \right | \ge m \cdot \epsilon \right ] &\le e^{-\frac{\left (\frac{\epsilon}{\p(B)} \right )^2}{2+\frac{\epsilon}{\p(B)}} m \p(B)} + e^{-\frac{\left (\frac{\epsilon}{\p(B)}\right )^2}{2}m\p(B_i)}\\
& \le e^{-\frac{\epsilon^2 m}{2\p(B) + \epsilon}} + e^{-\frac{\epsilon^2 m}{2}}\\
& \le 2e^{-\frac{\epsilon^2 m}{3}},
\end{align*} 
which is $\le \delta$ as long as $m \ge \frac{3 \ln(2/\delta)}{\epsilon^2}.$
\end{proof}
With Lemma \ref{lem:probEst} in hand, we proceed to the proof of Theorem \ref{thm:const}.
\begin{proof}[Proof of Theorem \ref{thm:const}]
For convenience, define $\gamma_1 =\frac{\alpha_1}{3}$,
 and $b = \lceil \log_{1-\frac{\gamma_1}{8}} \frac{\gamma_1}{8N} \rceil$. Note that  $b = O \left (\frac{\log \frac{N}{\alpha_1}}{\alpha_1} \right )$.
For $i \in \left \{1,\ldots, b \right \}$, define: 
$$\B_i = \left \{x : q_x \in \big (\left (1-\frac{\gamma_1}{8}\right)^i,\left(1-\frac{\gamma_1}{8}\right )^{i-1} \big ] \right \}.$$
Let $\B_{b+1} = \left \{x: q_x \le \left (1-\frac{\gamma_1}{8} \right )^{b}\right \}$.\footnote{Note that this this is different that the usual definition of $B_t = \{x: q_x = t\}$, but it is still within the same spirit of bucketing the elements based on their $q_x$ values.}
Note that $\B_1 \cup \ldots \cup \B_b \cup \B_{b+1} = \mathcal{X}$. Now, via Lemma \ref{lem:probEst}, with $O \left (\frac{b^2 \cdot \log b/\delta}{\alpha_2^2 \cdot \alpha_1^2} \right )$ samples from $\p$ it is possible to compute $\tilde \p(\B_1),\ldots, \tilde \p(\B_{b+1})$ such that, with probability $\ge 1-\delta$, 
$$|\p(\B_i) - \tilde \p(\B_i) | \le \frac{\gamma_1 \cdot \alpha_2}{8(b+1)}$$
for all $i$ simultaneously. Let $\mathcal{E}$ be the event that these approximations hold, and assume that $\mathcal{E}$ occurs. Then for any $i$ with $\tilde \p(\B_i) \le \frac{\alpha_2}{4(b+1)}$, it must be that 
\begin{align}\label{eq:LBound}
\p(\B_i) \le \frac{\alpha_2}{4(b+1)} + \frac{\gamma_1 \cdot \alpha_2}{8(b+1)} \le \frac{\alpha_2}{2(b+1)}.
\end{align}
 Let $L \subseteq \{1,\ldots, b+1\}$ be the set of all such $i$. Similarly, for $i$ with $\tilde \p(\B_i) > \frac{\alpha_2}{4(b+1)}$, it must be that:
 \begin{align}\label{eq:HBound}
 \p(B_i) > \frac{\alpha_2}{4(b+1)} - \frac{\gamma_1 \cdot \alpha_2}{8(b+1)} > \frac{\alpha_2}{8(b+1)}.
 \end{align}
  Let $H = \{1,\ldots, b+1\}\setminus L$ be the set of all such $i$. 

Define $\bv{w}$ as follows: for $x \in \cup_{i \in L} \B_i$ set $w_x = \frac{\alpha_2}{2\left | \cup_{i \in L} \B_i \right |}$. For $i \in H$, for $x \in \B_i$ let $w_x = \frac{\tilde \p(\B_i)}{|\B_i|}$. We have the following facts about $\w$:
\begin{enumerate}
\item For $i \in H$, $\bv{w}(\B_i) = \tilde \p(\B_i) \in \bv{p}(\B_i) \pm \frac{\gamma_1 \cdot \alpha_2}{8(b+1)}$, which by  the fact that $\p(\B_i) \ge  \frac{\alpha_2}{8(b+1)}$ (equation \eqref{eq:HBound}) gives for all $i \in H$:
\begin{align}\label{eq:wcal1}
\w(\B_i) \in \left (1\pm \gamma_1\right ) \p(\B_i).
\end{align} 
\item $\w(\cup_{i \in L} B_i) = \frac{\alpha_2}{2}$ and by \eqref{eq:LBound}, $\p(\cup_{i \in L} \B_i) = \sum_{i \in L} \p(\B_i) \le (b+1) \cdot \frac{\alpha_2}{2(b+1)} = \frac{\alpha_2}{2}$.
\end{enumerate} 
In combination, the above facts give that $\norm{\bv w}_1 \in (1\pm\gamma_1)$. Thus, letting $\q' = \frac{1}{\norm{\bv w}_1} \cdot \bv w$, we have:
\begin{enumerate}
\item Applying \eqref{eq:wcal1}, for all $i \in H$, $\left (\frac{1-\gamma_1}{1+\gamma_1}\right ) \p(\B_i) \le \q'(\B_i)\le \left (\frac{1+\gamma_1}{1-\gamma_1}\right ) \p(\B_i)$. Since $\gamma_1 =  \frac{\alpha_1}{3}$ we have $\frac{1-\gamma_1}{1+\gamma_1} \ge 1 -\alpha_1$ and $\frac{1+\gamma_1}{1-\gamma_1} \le 1 +\alpha_1$, which gives for all $i \in H$:
\begin{align}\label{eq:qcal1}
\q'(\B_i) \in \left (1\pm \alpha_1 \right ) \p(\B_i).
\end{align}
\item $\q'(\cup_{i \in L} \B_i) \ge \frac{1}{1+\gamma_1} \cdot \frac{\alpha_2}{2} \ge (1-\alpha_1) \cdot \frac{\alpha_2}{2} \ge (1-\alpha_1) \cdot \p(\cup_{i \in L} \B_i)$. Additionally, $\q'(\cup_{i \in L} \B_i) \le \frac{1}{1-\gamma_1} \cdot \frac{\alpha_2}{2} \le \alpha_2$.
\item $\norm{\q' - \w}_1 \le \gamma_1.$
\end{enumerate}
Properties (1) and (2) together give that $\q' \in \mathcal{C}(\p,\alpha_1,\alpha_2)$ where we define the set $T$ to be $\{\bar q_x \}$ for $x \in \cup_{i \in L} \B_i$.
Recalling that $b = O \left ( \frac{\log N/\alpha_1}{\alpha_1}\right )$,
the overall sample complexity used to construct $\q'$ is:
\begin{align*}
O \left (\frac{b^2 \cdot \log b/\delta}{\alpha_2^2 \cdot \alpha_1^2} \right ) = O\left (\frac{\log(N/\alpha_1)^2 \log(b/\delta)}{\alpha_1^4 \cdot \alpha_2^2} \right ) = O\left (\frac{\log(N/\alpha_1)^2 \cdot \log\left (\frac{\log N}{\delta \alpha_1}\right)}{\alpha_1^4 \cdot \alpha_2^2} \right ).
\end{align*}

Finally, it remains to show that if $\q \in \C(\p,\alpha_1,\alpha_2)$, then $\norm{\q-\q'}_1 \le O(\alpha_1 + \alpha_2)$.

For every $j \le b$,  since $\q$ places all probabilities within $(1\pm \frac{\gamma_1}{8}) = (1 \pm \frac{\alpha_1}{24})$ of each other on this bucket, for every $x \in \bar B_j$, $q_x \in (1 \pm \frac{\alpha_1}{24}) \cdot \frac{\q(\bar B_j)}{|\bar B_j|}$. We thus have:
\begin{align*}
\sum_{x \in \bar B_j} |q_x - q'_x| \le |\q(\bar B_j)-\q'(\bar B_j)| + O(\alpha_1)\cdot \q(\bar B_j).
\end{align*}
For $\bar B_{b+1}$ since $\q(\bar B_{b+1}) \le \frac{\alpha}{24}$, we simply have $\sum_{x \in \bar B_b+1} |q_x - q'_x| \le |\q(\bar B_j)-\q'(\bar B_j)| + O(\alpha_1)$. Thus overall:
\begin{align*}
\norm{\q-\q'}_1 =  \sum_{j = 1}^{b+1} \sum_{x \in \bar B_j} |q_x - q'_x| \le \sum_{j=1}^{b+1} |\q(\bar B_j) - \q'(\bar B_j)| + O(\alpha_1).
\end{align*}
We now bound the above sum using that  both $\q$ and $\q'$ are in $\C(\p,\alpha_1,\alpha_2)$. 
Let $T$ be the set of probabilities for which $\q$ may significantly overestimate $\p$ but places mass $\le \alpha_2$. Let $T'$ be analogous set for $\q'$ (see Definition \ref{def:calibrated2}). Let $\bar \q$ be vector obtained by setting $q_x = p_x$ for $\{x : q_x \in T\}$. Let $\bar \q'$ be defined analogously for $\q'$. We have:
\begin{align*}
\norm{\q-\q'}_1 \le \sum_{j=1}^{b+1} |\q(\bar B_j) - \q'(\bar B_j)| + O(\alpha_1) \le \sum_{j=1}^{b+1} |\bar \q(\bar B_j) - \bar \q'(\bar B_j)| + O(\alpha_1 + \alpha_2).
\end{align*}
Additionally, we can see that both $\bar \q$ and $\bar \q'$ are calibrated up to error $(1 \pm \alpha_1)$ on all $\bar B_j$ ($\bar \q$ is calibrated up to this error on all its level sets, which form a refinement of $\{\bar B_j\}$.) Thus we have:
\begin{align*}
\norm{\q-\q'}_1  \le \sum_{j=1}^{b+1} O(\alpha_1) \cdot \p(\bar B_j)+ O(\alpha_1 + \alpha_2) = O(\alpha_1+\alpha_2).
\end{align*}
which completes the claim.
\end{proof}

\subsection{Strong Properness Under Approximate Calibration}

We now show that Theorem \ref{thm:nice-strong-proper} is robust to approximation calibration, using a similar proof strategy. See Table \ref{fig:appStrong} for a sampling of results that this implies, which essentially match those given by Table \ref{tab:intro} in the case of exact calibration.
\begin{theorem} \label{thm:nice-strong-properApprox}
  Suppose $\ell(\q,x) = f(\frac{1}{\qx})$ where $f$ is non-decreasing, and for $z \ge \frac{1}{\max_x q_x}$ is non-negative and satisfies $f'(z) \le \frac{D(z)}{z}$ for some non-decreasing function $D$. Also suppose that $f$  is $\frac{C(z)}{z^2}$-left-strongly concave for $C$ that is non-increasing and non-negative for  $z \ge 1$. Then for all $p \in \Delta_{\X}$, $\alpha_1 \le 1/2$ and $\q \in  \mathcal{C}(\p, \alpha_1,\alpha_2)$:
\begin{align*}
\ell(\q;\p) - \ell(\p;\p) \ge \frac{C\left (\frac{N}{2\alpha_2} \right )}{32} \cdot  \left (\norm{\p-\q}_1 - \alpha_1 -  5\alpha_2\right )^2 - 2\alpha_1  \cdot D \left (\frac{N}{2\alpha_2}\right ) - 3\alpha_2 \cdot f \left (\frac{N}{3\alpha_2}\right ).
\end{align*}
\end{theorem}
\begin{proof}
  Let $\q \in \mathcal{C}(\p, \alpha_1,\alpha_2)$ be piecewise uniform with pieces {$\{B_t\}_{t\in T(\vec q)}$.} 
   Let  {$L_1 = \left\{t : \frac{\p(B_t)}{|B_t|} \le \frac{\alpha_2}{N} \right\}$.} 
   Let  {$H  \subseteq T(\vec q) \setminus L_1$ contain all remaining $t$ for which $\q(B_t) \in (1 \pm \alpha_1) \p(B_t)$. Finally, let $L_2 = T(\q) \setminus (H \cup L_1)$ contain all remaining $t\in T(\q)$.}
    Let  {$\epsilon_t = \sum_{x \in B_t} |\px - \qx|$,} with  {$\sum_{t\in T(\q)} \epsilon_t = \epsilon = \norm{\p - \q}_1$.} Finally, consider $\q' \in \C(\p)$ that is exactly calibrated and piecewise uniform on  {$B_t(\vec q)$, that is, $q'_x = \p(B_t(\q)) / |B_t|$ for all $x\in B_t(\vec q)$ and $t\in T(\vec q)$.}

By definition of $L_1$ we have  {$\p(\cup_{t \in L_1} B_t) = \q'(\cup_{t \in L_1} B_t) \le \alpha_2.$} Additionally, by our definition of approximate calibration, {for any $t\in L_1$, either $\q(B_t) \in (1\pm \alpha_1) \p(B_t)$ or  else $t\in T$ is in the set of buckets for which the total mass $\q(\cup_{t\in T} B_t) \leq \alpha_2$.} We have

{\begin{align*}
  \q(\cup_{t \in L_1} B_t) \le (1+\alpha_1)\alpha_2 + \alpha_2 \le 3 \alpha_2.
  \end{align*} 
}

Similarly, using the definition of approximate  calibration we have: 
{\begin{align*}
  \q(\cup_{t \in L_2} B_t) \le \alpha_2 \text{ and }\p(\cup_{t \in L_2} B_t) =  \q'(\cup_{t \in L_2} B_t)\le \frac{\alpha_2}{1-\alpha_1} \le 2\alpha_2.
  \end{align*}
  }
This gives us that the truly calibrated  $\q'$ is close to the approximately calibrated $\q$:
\begin{align*}
\norm{\q - \q'}_1 &\le \sum_{t \in H} \alpha_1 \cdot \p(B_t) + \sum_{t \in L_1 \cup L_2} \left (\q(B_t) + \q'(B_t)) \right )\le \alpha_1 + 5 \alpha_2.
\end{align*}
Thus, by triangle inequality we have
\begin{align}\label{qprimeNorm}
\norm{\p - \q'}_1 \ge \norm{\p-\q}_1 -\alpha_1 - 5\alpha_2.
\end{align}
We can thus bound $\ell(\q'; \p) - \ell(\p; \p)$ following the proof of Theorem \ref{thm:nice-strong-proper}.
Let $\epsilon' = \norm{\p - \q'}_1$ and $\epsilon'_t = \sum_{x \in B_t} |p_x -q'_x|$.
 Let $\ell_H(\q;\p) = \sum_{j \in H} \sum_{x \in B_t} p_x f\left (\frac{1}{q_x}\right )$ be the loss restricted to the buckets in $H$. By \eqref{eq:jenseOut2} we can bound:
\begin{align*}
\ell(\q'; \p) - \ell(\p; \p) \ge \ell_H(\q'; \p) - \ell_H(\p; \p) \ge \sum_{t \in H} \pof{B_t} \frac{b(\frac{|B_t|}{\q'(B_t)})}{32} \frac{(\epsilon'_t)^2}{\left (\frac{\q'(B_t)}{|B_t|}\right)^2 \pof{B_t}^2}.
\end{align*}
Since $H$ excludes call elements in $L_1$, for all $t \in H$, $\frac{\q'(B_t)}{|B_t|} \ge \frac{\alpha_2}{N}$. Thus by our assumption on $b(\cdot)$:
\begin{align*}
\ell_H(\q'; \p) - \ell_H(\p; \p) \ge \sum_{t \in H} \frac{C \left (\frac{N}{\alpha_2}\right)}{32} \frac{(\epsilon_t')^2}{\pof{B_t}}.
\end{align*}
and applying the same argument as in Theorem \ref{thm:nice-strong-proper} can lower bound this quantity using \eqref{qprimeNorm}  by:
\begin{align}\label{eq:lhBound}
\ell_H(\q'; \p) - \ell_H(\p; \p) \ge \frac{C\left (\frac{N}{\alpha_2} \right )}{32} \cdot  \left (\norm{\p-\q}_1 - \alpha_1 -  5\alpha_2 \right )^2.
\end{align}

We next show that $\ell_H(\q'; \p) - \ell_H(\q; \p) $ is not too large. Since $\q$ and $\q'$ are both piecewise uniform on $\{B_t\}_{t\in T(\vec q)}$ and since $\q'$ is calibrated (i.e, $\q'(B_t) = \p(B_t)$ for all $t$), 
\begin{align*}
\ell_H(\q'; \p) - \ell_H(\q; \p) = \ell_H(\q'; \q') - \ell_H(\q; \q').
\end{align*}
We have using that $f$ is nondecreasing:
\begin{align}\label{eq:qPrimeBound}
 \ell_H(\q'; \q') &= \sum_{t \in H} \sum_{x \in B_t} \q'(B_t) \cdot f \left (\frac{|B_t|}{\q'(B_t)} \right )\le \sum_{t \in H} \sum_{x \in B_t} \q'(B_t) \cdot f \left (\frac{1}{(1-\alpha_1) \cdot q_x} \right )
\end{align}
Using the concavity of $f$ along with the assumption that $f'(z) \le \frac{D(z)}{z}$, we have:
\begin{align*}
f \left (\frac{1}{(1-\alpha_1) \cdot q_x} \right ) &\le f \left ( \frac{1}{q_x} \right ) + f' \left (\frac{1}{q_x} \right  ) \cdot \left (\frac{1}{(1-\alpha_1)q_x} - \frac{1}{q_x} \right )\\
& \le f\left (\frac{1}{q_x}\right ) + D\left (\frac{1}{q_x}\right ) \cdot q_x \cdot \frac{\alpha_1}{(1-\alpha_1) q_x}\\
& \le f\left (\frac{1}{q_x}\right )  + D\left (\frac{1}{q_x}\right ) \cdot 2\alpha_1.
\end{align*}
Plugging back into \eqref{eq:qPrimeBound}, using that $q_x \ge (1-\alpha)q'_x \ge \frac{\alpha_2(1-\alpha_1)}{ N} \ge \frac{\alpha_2}{2N}$ for all $x \in \cup_{t \in H} B_t$ we have:
\begin{align*}
 \ell_H(\q'; \q') &\le \sum_{t \in H} \sum_{x \in B_t} 
 \vec q'(B_t) 
 \left [f\left (\frac{1}{q_x}\right ) + D \left (\frac{N}{2\alpha_2}\right )\cdot 2\alpha_1 \right ]\\
 &\le \ell_H(\q;\q') + D \left (\frac{N}{2\alpha_2}\right ) \cdot 2\alpha_1. 
\end{align*}
Combined with \eqref{eq:lhBound} this gives:
\begin{align}\label{eq:justHBound}
\ell_H(\q;\p) - \ell_H(\p;\p) \ge  \frac{C\left (\frac{N}{2\alpha_2} \right )}{32} \cdot  \left (\norm{\p-\q}_1 -  \alpha_1 - 5\alpha_2 \right )^2 - 2\alpha_1 \cdot D \left (\frac{N}{2\alpha_2}\right ).
\end{align}
Finally, let $\ell_L(\q;\p)$ be the loss restricted to buckets in $L_1 \cup L_2$. As shown, $\sum_{t \in L_1 \cup L_2} \sum_{x \in B_t} p_x \le 3\alpha_2$. By the concavity of $f(z)$ we thus have:
\begin{align*}\label{rLogBound}
\ell_L(\p;\p) = \sum_{t \in L_1 \cup L_2} \sum_{x \in B_t} p_x \cdot f \left ( \frac{1}{p_x} \right ) \le 3\alpha_2 \cdot f \left (\frac{N}{3\alpha_2}\right ).
\end{align*}
Combined with \eqref{eq:justHBound} this finally gives:
\begin{align*}
\ell(\q;\p) - \ell(\p;\p) & \ge \ell_H(\q;\p) - \ell_H(\p;\p) -\ell_L(\p;\p)\\
&\ge \frac{C\left (\frac{N}{2\alpha_2} \right )}{32} \cdot  \left (\norm{\p-\q}_1 - \alpha_1 -  5\alpha_2\right )^2 - 2\alpha_1  \cdot D \left (\frac{N}{2\alpha_2}\right ) - 3\alpha_2 \cdot f \left (\frac{N}{3\alpha_2}\right ),
\end{align*}
which completes the theorem.
\end{proof}

\begin{table}[h]
\centering
\small
\begin{tabular}{c|c|c|c|c|c|c}
  $\ell(\q,x)$  & $f(z)$ & $D(z)$ & $C(z)$ &  $\alpha_1$ & $\alpha_2$ & $\frac{\ell(\q;\p)-\ell(\p;\p)}{\epsilon^2}$  \\
  \hline
$\ln \frac{1}{\qx}$ & $\ln(z)$ & $1$ & $1$   & $ \Theta \left (\epsilon^2 \right )$ & $\Theta \left (\frac{\epsilon^2}{\ln N} \right )$ &$\Omega(1)$  \\
     $\ln\frac{1}{\qx}^p$, $p \in (0,1]$ & $\left (\ln(z)\right )^p$ & $1$& $\ln(z)^{p-1}$  & $\Theta \left (\epsilon^2 \right )$ & $\Theta \left (\frac{\epsilon^2}{(\ln N)^{p}} \right )$ & $\Omega \left ( \left (\ln N \right )^{p-1}\right )$\\
  $\ln\left(\ln \frac{1}{\qx} \right)$  & $\ln(\ln(z))$ & $1$ & $1/\ln(z)$ &$\Theta \left (\epsilon^2\right )$ & $\Theta \left (\frac{\epsilon^2}{\ln(\ln N)} \right )$& $\Omega \left ( \frac{1}{\ln N} \right )$\\
    $\frac{1}{\sqrt{\qx}}$      & $\sqrt{z}$ & $2 \sqrt{z}$ & $\frac{1}{4 \sqrt{z}}$  &$\Theta \left (\frac{\epsilon^{4}}{N} \right )$ & $\Theta \left (\frac{\epsilon^{4}}{N} \right )$& $\Omega \left (\frac{\epsilon^2}{N}\right )$  \\
 $\left(\ln\frac{e^2}{\qx}\right)^2$        & $\ln(e^2 z)^2$ & $2 \ln(z)+2$ & $2$ & $\Theta \left (\frac{\epsilon^2}{\ln N} \right )$ & $\Theta \left (\frac{\epsilon^2}{(\ln N)^2} \right )$ & $\Omega \left (1 \right )$
\end{tabular}
 \caption{Examples of loss functions that are strongly proper over $\C(\p,\alpha_1,\alpha_2)$.  
  We let $\epsilon \eqdef \norm{\p-\q}_1$ and assume for simplicity that $\epsilon \ge 1/N$.
  We fix values of $\alpha_1$ and $\alpha_2$ that yield a strong properness bound nearly matching that of Theorem \ref{thm:gen} for truly calibrated distributions. Note that in the theorem $D(z)$ is required to be nondecreasing and thus we set it to $1$ for all loss functions considered that grow slower than the log loss.}\label{fig:appStrong}
    \normalsize
  \end{table}


\subsection{Concentration Under Approximate Calibration}

It is also easy to show that our main concentration result, Theorem \ref{thm:gen}, is robust to approximate calibration, since this result just uses that calibration ensures $\frac{q_x}{p_x}$ is not too small for any $x$ (Lemma \ref{lem:calBound}). In particular, using an identical argument to what is used in Lemma \ref{lem:calBound} we can see from Definition \ref{def:calibrated2} that for $\q \in \C(\p,\alpha_1,\alpha_2)$, for all $x$, $q_x \ge \frac{(1-\alpha_1)p_x}{N} \ge \frac{p_x}{2N}$ for $\alpha_1 \le 1/2$. Following the proof of  Theorem \ref{thm:gen} using this bound in place of Lemma \ref{lem:calBound} gives:

\begin{theorem}\label{thm:genApprox}
  Suppose $\ell$ is a local loss function with $\ell(\q,x) = f\left(\frac{1}{\qx}\right)$ for non-negative, non-decreasing, concave $f(z)$. Suppose further that $f(z) \le c\sqrt{z}$ for all $z \ge 1$ and some constant $c$.
  Then $\ell$ concentrates over $\C(\p,\alpha_1,\alpha_2)$ for any $\alpha_1 \le 1/2$ and $m(\gamma,\delta,N) \le N$ satisfying 
  $$m(\gamma,\delta,N) \ge \frac{c_1 \cdot f\left(\beta\right)^2 \ln\frac{1}{\delta}}{\gamma^2},$$
  where $c_1$ is a fixed constant and $\beta \eqdef \frac{{32} N^{8}}{\delta \cdot \min(1,\gamma^2/c^2)}$.
  
  That is, for any {$\p \in \Delta_\X, \q \in \C(\p, \alpha_1, \alpha_2)$,} drawing at least $  m(\gamma,\delta,N)$ samples guarantees $|\ell(\q;\ph) - \ell(\q;\p)| \leq \gamma$ with probability $\ge 1-\delta$.
\end{theorem}

First, the analogue of Lemma \ref{lem:calBound}.
\begin{lemma} \label{lem:approx-calBound}
  For all $\p$ and all $\q \in \C(\p,\alpha_1,\alpha_2)$ with $\alpha_1 \leq 1/2$, for all $x$, we have $\qx \geq \frac{\px}{N(1-\alpha_1)} \geq \frac{\px}{2N}$.
\end{lemma}
\begin{proof}
  Given $x$, let $B = \{x' : \qat{x'} = \qx\}$. By calibration,
  \[ \qx = \frac{\qof{B}}{|B|} \geq \frac{\qof{B}}{N} \geq \frac{(1-\alpha_1)\pof{B}}{N} \geq \frac{(1-\alpha_1)\px}{N} . \]
  If $\alpha_1 \leq 1/2$, we get $\qx \geq \frac{\px}{2N}$.
\end{proof}

\begin{proof}[Proof of Theorem \ref{thm:genApprox}]
  By Lemma \ref{lem:approx-calBound}, we have $\qx \geq \frac{c_2 \px}{N}$ for all $x$ with $c_2 = 0.5$.
  We apply Proposition \ref{prop:gen}, with all parameters exactly as in Theorem \ref{thm:gen} except with $c_2 = 0.5$ rather than $1$.
\end{proof}

Note that Theorem \ref{thm:genApprox} is essentially identical to Theorem \ref{thm:gen}, up to a constant factor in $\beta$. Thus, all of our {concentration results} hold, up to constant factors, when {$\q \in \C(\vec p, \alpha_1,\alpha_2)$} for $\alpha_1 \le 1/2$ and \emph{any $\alpha_2$.}
  {Also note that Theorem~\ref{thm:genApprox} gives a high probability bound for any $\q \in C(\p)$. If for example, we wish to minimize $\ell(\q;\p)$ over some set of candidate calibrated distributions, we could form an $\epsilon$-net over these distributions and apply the theorem to all elements of this net, union bounding to obtain a bound on the probability  that the empirical loss is close to the true loss on all elements. Optimizing would then yield a distribution with loss within $\gamma$ of the minimal.}

\subsection{Sample Properness Under Approximate Calibration}

Finally, we note that we can obtain a sample properness result under approximate calibration by combining Theorems \ref{thm:nice-strong-properApprox} and \ref{thm:genApprox} (analogously to how Theorem \ref{thm:fs-diff} is proven using Theorems \ref{thm:nice-strong-proper} and \ref{thm:gen}).

\begin{theorem} \label{thm:fs-diffApprox}
 Suppose $\ell$ is a local loss function with 
 $\ell(\q,x) = f(\frac{1}{\qx})$ for nonnegative, increasing, concave $f(z)$. Suppose further that $f(z) \le c\sqrt{z}$ for all $z \ge 1$, that $f'(z) \le \frac{D(z)}{z}$ for some non-decreasing function $D$, and that, for some constant $c$, $f$ is $\frac{C(x)}{x^2}$-left-strongly concave for  where $C(x)$ is nonincreasing and nonnegative for $x \ge 1$. Then for all $\p \in \Delta_\X$ and $\q \in \mathcal{C}(\p, \alpha_1,\alpha_2)$ with $\alpha_1, \alpha_2 \le \frac{\norm{\p-\q}_1^2}{12}$, if $\ph$ is the empirical distribution constructed from $m$ independent samples of $\p$ with $m \le N$ and 
 \begin{align*}
 m \ge \frac{c_1 \cdot f(\beta)^2 \ln \frac{1}{\delta}}{\left (C \left (\frac{N}{2\alpha_2} \right ) \norm{\p-\q}_1^2\right )^2},
 \end{align*}
 where $c_1$ is constant and  \small$\beta \eqdef \frac{576 N^{8}}{\delta \cdot \min \left (1,\left [C\!\left(\frac{2N}{\alpha_2}\right) \cdot \frac{ \norm{\p-\q}_1^2}{128c} \right]^2 \right )}$\normalsize, then with prob. $\ge 1-\delta$:
  $$\ell(\q;\ph)- \ell(\p;\ph) > C\!\left(\frac{N}{2\alpha_2}\right) \cdot \frac{ \norm{\p-\q}_1^2}{384}- 2\alpha_1  \cdot D \left (\frac{N}{2\alpha_2}\right ) - 3\alpha_2 \cdot f \left (\frac{N}{3\alpha_2}\right ).$$
   \end{theorem}
   Note that the right hand side of the above inequality will generally be positive (giving us our desired sample properness guarantee) if we set $\alpha_1$ and $\alpha_2$ small enough. See Table \ref{fig:appStrong} for examples of how these parameters can be set for a variety of loss functions.
   \begin{proof}
Applying Theorem \ref{thm:nice-strong-properApprox} and the assumption that $\alpha_1,\alpha_2 \le  \frac{\norm{\p-\q}_1^2}{12}$ we have:
\begin{align}\label{layoutTheGammaApprox}
\ell(\q;\p) - \ell(\p;\p) &\ge \frac{C\left (\frac{N}{2\alpha_2} \right )}{32} \cdot  \left (\norm{\p-\q}_1 - \alpha_1 -  5\alpha_2\right )^2 - 2\alpha_1  \cdot D \left (\frac{N}{2\alpha_2}\right ) - 3\alpha_2 \cdot f \left (\frac{N}{3\alpha_2}\right )\nonumber\\
&\ge \frac{C\left (\frac{N}{2\alpha_2} \right )}{128} \cdot  \norm{\p-\q}_1^2 - 2\alpha_1  \cdot D \left (\frac{N}{2\alpha_2}\right ) - 3\alpha_2 \cdot f \left (\frac{N}{3\alpha_2}\right ).
\end{align}
Let $\gamma \eqdef C\!\left(\frac{N}{2\alpha_2}\right) \cdot \frac{ \norm{\p-\q}_1^2}{128}$.
Additionally, since $f(x) \le c \sqrt{z}$ for $z \ge 1$ and since $\q, \p \in \mathcal{C}(\p,\alpha_1,\alpha_2)$, applying Theorem \ref{thm:gen} with error parameter $\gamma/3$ and failure parameter $\delta/2$, we have for
$\beta \eqdef \frac{576 N^{8}}{\delta \cdot \min(1,\gamma^{2}/c^2)}$, if $m \ge \frac{c_1 f(\beta)^2 \lg \frac{2}{\delta}}{(\gamma/3)^2}$ for large enough constant $c_1$ then the following hold, each with probability  $\ge 1-\delta/2$:
\begin{align*}
|\ell(\q;\ph) - \ell(\q;\p)| \leq \frac{\gamma}{3} \text{ and } |\ell(\p;\ph) - \ell(\p;\p)| \leq \frac{\gamma}{3}.
\end{align*}
By a union bound, with probability $\ge 1-\delta$ both bounds hold simultaneously and by \eqref{layoutTheGammaApprox} we have:
\begin{align*}
\ell(\q;\ph) - \ell(\p; \ph) \ge \ell(\q;\p) - \ell(\p;\p) - \frac{2\gamma}{3} \ge \frac{\gamma}{3} - 2\alpha_1  \cdot D \left (\frac{N}{2\alpha_2}\right ) - 3\alpha_2 \cdot f \left (\frac{N}{3\alpha_2}\right ).
\end{align*}
This completes the theorem. Plugging the value of $\gamma$ in we see that the bound holds for
\begin{align*}
m \ge \frac{c_1 f(\beta)^2 \ln \frac{1}{\delta}}{\left (C\!\left(\frac{2N}{\alpha_2}\right) \cdot \frac{ \norm{\p-\q}_1^2}{128}/3\right )^2} = \frac{c_1' f(\beta)^2 \ln \frac{1}{\delta}}{\left (C\!\left(\frac{2N}{\alpha_2}\right) \cdot \norm{\p-\q}_1^2\right )^2} 
\end{align*} 
for large enough constant $c_1'$. Additionally, we see that:
\begin{align*}
\beta = \frac{576 N^{8}}{\delta \cdot \min \left (1,\left [C\!\left(\frac{2N}{\alpha_2}\right) \cdot \frac{ \norm{\p-\q}_1^2}{128c} \right]^2 \right )}.
\end{align*}
\end{proof}


\section{Details on Motivating Example}\label{app:experiments}

We now give details on the motivating example for considering alternatives to the log loss in the introduction (see Figure \ref{fig:lang}.)

\paragraph{Dataset:} Our primary data set  is a list of $36663$ of the most frequent English words, along with their frequencies in a count of all books on Project Gutenberg \cite{wordlists}. We then obtained a list of the $10000$ most frequent French \cite{french} and German \cite{german} words. All capitals were converted to lower case, all accents removed, and all duplicates from the French and German lists removed. After preprocessing, the data consisted of  the original  $36663$ English words along with $16409$ French/German words. We gave the French and German words uniform frequency values, with the total frequency of these words comprising $12.23\%$ of the probability  mass of the word distribution. 

Our tests {are relatively insensitive to the exact frequency chosen for the French/German words within the reasonable range of $5$-$30\%$. Low frequency ($< 5\%$ of the total probability mass) is not sufficient noise to make the log loss minimizing distribution to perform poorly. On the other hand, high frequency ($> 30\%$  of the total probability mass) is too large and forces even our loglog loss minimizing distribution to perform poorly --due to its poor performance on the French and German words.}

\paragraph{Learning $\q_1$ and $\q_2$:}

We trained the candidate distribution $\q_1$ by minimizing log loss for a basic character trigram model. Minimizing log loss here simply corresponds to setting the trigram probabilities to their relative frequencies in the dataset. These frequencies were computed via a scan over all words in the dataset, taking into account the word frequencies. Note that we have full access to the target $\p$ and thus $\q_1$ exactly minimizes $\ell(\q;\p) = \E_{x \sim \p} \left [\ln \frac{1}{q_x} \right ]$ over all trigram models.

We trained $\q_2$ by  distorting the optimization to place higher weight on the head of the distribution. In particular, we let $\bar \p$ be the distribution with $\bar p_x \propto p_x^{\alpha}$ for $\alpha = 1.4$. and minimized log loss over $\bar \p$. We saw similar performance for $\alpha \in [1.3,2]$. Below this range, there was not significant difference between $\q_1$ and $\q_2$. Above this range, $\q_2$ placed very  large mass on the head of the distribution, e.g., outputting the most common word \texttt{the} with probability $\ge .40$.

\paragraph{Results:} Our results are summarized in Figure \ref{fig:lang}. We can see that $\q_2$ seems to give more natural word samples and, while it achieves worse log loss than $\q_1$ (it must since $\q_1$ minimizes this loss over all trigram models), it achieves better log log loss. This indicates that in this setting, the log log loss may be a more appropriate measure to optimize. Our approach to training $\q_2$ via a reweighting of $\p$ can be viewed a heuristic for minimizing log log loss. Developing better algorithms for doing this, especially under  the constraint that $\q_2$  is (approximately) calibrated is an interesting direction. 

One way to see the improved performance of $\q_2$ is that its cumulative distribution more closely matches that of $\p$. See plot in Figure \ref{fig:lang}. Overall $\p$ places $87.77\%$ of its mass on the English words in the input distribution. $\q_1$ places $45.56\%$ of its mass on these words and $\q_2$ places $83.40\%$ of its mass on them.
Note that the cumulative distribution plot and these statistics are \emph{deterministic}, since $\q_1$ and $\q_2$ are trained by exactly minimizing log loss over the distributions $\p$ and $\bar \p$ without sampling. Thus no error bars are shown. 

Below we show an extended sampling of words from $\q_1$, $\q_2$ and $\p$, evidencing $\q_2's$ superior performance on the task of generating natural English words. In this single run, e.g., $\q_1$ generates 6 distinct commonly used English words $\{\texttt{and},\texttt{the},\texttt{why},\texttt{soon},\texttt{caps},\texttt{of}\}$. $\q_2$ generates 10: $\{\texttt{all},\texttt{the},\texttt{which},\texttt{on},\texttt{take},\texttt{and},\texttt{be},\texttt{in},\texttt{of},\texttt{he} \}$. $\p$ generates 19, all with the except of the German word \texttt{verweigert}. More quantitatively, in a run of $10000$ random samples, $\p$ generates $2497$ distinct English words (the word distribution is very  skewed so many duplicates of common words are generated). In comparison, $\q_1$ generates $815$ distinct words and $\q_2$ generates $957$.

Of course there are many methods of evaluating the performance of $\q_1$ and $\q_2$, which generally will be application specific. Our experiments are designed to give just a simple example, motivating the idea that minimizing log loss may not always be the optimal choice, and, like in classification and regression, there is room for alternative loss functions to be considered.  

$$\small
		\begin{array}{|c|c|c|}
		\hline
		\textbf{Samples from $\q_1$} & \textbf{Samples from $\q_2$} & \textbf{Samples from $\p$} \\
		\hline
		\texttt{and} & \texttt{all} & \texttt{old}                          \\
		\texttt{tiest} & \texttt{the} & \texttt{verweigert}\\
		\texttt{rike} & \texttt{which} & \texttt{five}\\
		\texttt{agal} & \texttt{nesell} & \texttt{common}\\
		\texttt{the} & \texttt{on}& \texttt{ny}\\
		\texttt{itunge} & \texttt{whostionespirs} & \texttt{significance}\\
		\texttt{cand} & \texttt{the} & \texttt{friend}\\
		\texttt{ho} & \texttt{take} & \texttt{i}\\
		\texttt{aren} & \texttt{the} & \texttt{with}\\
		\texttt{why} & \texttt{and} & \texttt{museum}\\
		\texttt{soon} & \texttt{be} & \texttt{the}\\
		\texttt{ca} & \texttt{frould} & \texttt{without}\\
		\texttt{caps} & \texttt{in} & \texttt{in}\\
		\texttt{der} & \texttt{the} & \texttt{ethan}\\
		\texttt{connestand} & \texttt{the} & \texttt{pointed}\\
		\texttt{of} & \texttt{goich} & \texttt{def}\\
		\texttt{per} & \texttt{of} & \texttt{down}\\
		\texttt{shicy} & \texttt{ithe} & \texttt{the}\\
		\texttt{theared} & \texttt{he} & \texttt{sky}\\
		\texttt{introt} & \texttt{ong} & \texttt{the}\\
		\hline
		\end{array}
\normalsize$$


\section{Calibration Definition} \label{app:calibrated}

In this section we give further discussion on our definition of calibration.
Most typically in forecasting, calibration is a property of a \emph{sequence} of forecasts $\q^{(1)},\dots,$ evaluated against a \emph{sequence} of samples $x^{(1)},\dots$.
So our definition may require some background.
First, we give a justification based on $\q$ as a coarsening of $\p$.
Then, we show how formalizations of calibration for sequences of forecasts can be related to our definition.

\paragraph{As a coarsening.}
One way to view the forecast $\q$ is as a coarsening of $\p$ in the sense of assigning probabilities to certain events $B_{\alpha} \subseteq \X$, but remaining agnostic as to the relative probabilities of various elements of $B_{\alpha}$, assigning all of them equal weight $\alpha$.
By dividing $\X$ into maximal pieces $B_{\alpha}$ on which $\q$ is piecewise uniform, in this way one obtains that $\q$ is literally a coarsening of $\p$ if $\pof{B_{\alpha}} = \qof{B_{\alpha}}$ for each piece (as the pieces partition $\X$).
This is our definition of calibration.

This directly captures the typical informal definition of calibration as ``events that are assigned probability $\beta$ occur a $\beta$-fraction of the time'', where the pieces $B_{\alpha}$ are the events and $\beta = \qof{B_{\alpha}} = \pof{B_{\alpha}}$ are the probabilities assigned to them.

It is also consistent with standard formalizations of calibration for sequences (see below), as if $x^{(s)} \sim \p$ i.i.d. each round and $\q^{(s)} = \q$ each round, one has that in the limit, each piece $B_{\alpha}$ will be represented as often as $\q$ predicts.

\paragraph{Sequences of forecasts.}
Calibration of sequences can be formalized, for example, as follows.
If each $x^{(t)} \in \X = \{0,1\}$, then we can let $R_t$ be the set of rounds $s \leq t$ where $x^{(s)} = 1$ and $S_t(\q)$ be the set of rounds $s \leq t$ where $\q^{(s)} = \q$.
In this case, the sequence is termed \emph{calibrated} if, on rounds where $\q$ was predicted, the fraction of times that $x^{(s)} = 1$ converges to $\qat{1}$:
  \[  \forall \q: \qquad \lim_{t\to\infty} \frac{|S_t(\q) \cap R_t|}{|S_t(\q)|}  = \qat{1} . \]

One way to obtain our definition is by ``flattening'' this one: let there be a finite number of rounds and suppose $\p,\q$ are probability distributions over rounds (so $\p$ will pick exactly one round to occur, and $\q$ assigns a binary prediction to each round).
In this case we can let $S(\alpha) = \{t : \qat{t} = \alpha\}$ be the set of rounds assigned a probability $\alpha$ by the forecast, then naturally the round $t \sim \p$ lies in this set with probability $\pof{S(\alpha)}$.
So the flattened definition of calibration requires that for each $\alpha$, $\pof{S(\alpha)} = \qof{S(\alpha)}$, which is exactly our definition.

Our definition can also be obtained as described above by letting $\X$ be general, letting $\q$ be forecast on each round while $x^{(s)} \sim \p$ i.i.d. each round.
If one interprets $\q$ as a distribution over events $B_{\alpha}$ that partition $\X$, one obtains the requirement that in the limit $\pof{B_{\alpha}} = \qof{B_{\alpha}}$ for each $\alpha$.

\section{Strong Properness in $\ell_2$ Norm} \label{app:strong-l2}
Our criteria can be extended to utilize different distance measures than our choice of $\ell_1$ or total variation distance.
However, justifying and investigating other measures requires further work.
In particular, this section shows why a choice of $\ell_2$ distance can be problematic.

Following our main definitions, one can define a loss to be strongly proper in $\ell_2$ if, for all $\p,\q$,
  \[ \ell(\q;\p) - \ell(\p;\p) \geq \frac{1}{2} \|\p-\q\|_2^2 . \]
In particular, consider the quadratic loss $\ell(\q,x) = \frac{1}{2} \norm{\pointx - \q}_2^2$, which can be shown to be $1$-strongly-proper in $\ell_2$ (Corollary \ref{cor:quadloss-sp-l2}).
However, the usefulness of this guarantee can be limited, as the following example shows.
\begin{proposition}
  Given a $1$-strongly proper loss in $\ell_2$ norm, $\q$ can assign probability zero to the entire support of $\p$, yet have expected loss within $\frac{2}{N}$ of optimal.
\end{proposition}
\begin{proof}
  Let $\X = \{1,\dots,N\}$ for even $N$.
  Let $\p$ be uniform on $\{1,\dots, \frac{N}{2}\}$ and let $\q$ be uniform on $\{\frac{N}{2}+1,\dots,N\}$.

  The point is that for any such ``thin'' distributions (small maximum probability), their $\ell_2$ norms $\norm{\p},\norm{\q}$ are vanishing and by the triangle inequality so is the distance $\norm{\p-\q}$ between them.
  
  In this example, $\norm{\p-\q}_2^2 = N\left(\frac{2}{N}\right)^2 = \frac{4}{N}$.
  So strong properness only guarantees that the difference in loss is $\ell(\q;\p) - \ell(\p;\p) \geq \frac{2}{N}$.
  In fact, this is exactly matched by the quadratic loss, where the difference in expected score (the Bregman divergence of the two-norm) is exactly $\frac{1}{2}\norm{\p-\q}_2^2 = \frac{2}{N}$.
\end{proof}
Thus, strongly proper losses in $\ell_2$ can converge to optimal expected loss at the rapid rate of $O(\tfrac{1}{N})$ even when making completely incorrect predictions.


\section{Strongly Proper Losses and Scoring Rules on the Full Domain} \label{app:strong-proper-simplex}
In this section, for completeness, we investigate the strongly proper criterion in the traditional setting of proper losses (equivalently, scoring rules).
The main result is that, just as (strictly) proper losses are Bregman divergences of (strictly) convex functions, so are \emph{strongly} proper losses Bregman divergences of \emph{strongly} convex functions.
We derive some non-local strongly proper losses.
These results may be of independent interest.

\paragraph{Terminology.}
Given a function $f: \R^d \to \R$, the vector $v \in \R^d$ is a \emph{supergradient} of $f$ at $z$ if for all $z'$, we have $f(z') \leq f(z) + v \cdot (z'-z)$.
(In other words, there is a tangent hyperplane lying above $f$ at $z$ with slope $v$.)
A function is \emph{concave} if it has at least one supergradient at every point.
(If exactly one, it is differentiable.)
In this case, use $df(z)$ to denote a choice of a supergradient of $f$ at $z$.

Given a concave $f$, the \emph{divergence function of $f$} is
  \[ D_{-f}(z,z') \eqdef \left[ f(z') + df(z') \cdot (z - z') \right] - f(z) , \]
the gap between $f(z)$ and the linear approximation of $f$ at $z'$ evaluated at $z$.
The reason for this notation is that $D_{-f}$ is the Bregman divergence of the convex function $-f$.

\begin{definition}[Strongly Concave]\label{def:generalstrongConcavity}
  A function $f:\R^d \to \R$ is \emph{$\beta$-strongly concave} with respect to a norm $\norm{\cdot}$ if for all $z,z'$,
    \[ D_{-f}(z,z') \geq \frac{\beta}{2}\norm{z-z'}^2  . \]
\end{definition}

\subsection{Background: proper loss characterization}
We first recall some background from theory of proper scoring rules, phrased in the loss setting.
Given a loss $\ell(\q,x)$, the expected loss function is $H_{\ell}(\p) = \ell(\p;\p)$.
The following classic characterization says that (strict) properness of $\ell$ is equivalent to (strict) concavity of $H_{\ell}$.
\begin{theorem}[\citep{mccarthy1956measures,savage1971elicitation,gneiting2007strictly}]
  $\ell$ is a (strictly) proper loss if and only if $H_{\ell}$ is (strictly) concave.
  If so, we must have
    \[ \ell(\q,x) = H_{\ell}(\q) + d H_{\ell}(\q) \cdot (\pointx - \q) \]
  where $d H_{\ell}(\q)$ is any supergradient of $H_{\ell}$ at $\q$ and $\pointx$ is the point mass distribution on $x$.
\end{theorem} 
\begin{corollary}
  The expected loss of $\q$ under true distribution $\p$ is the linear approximation of $H_{\ell}$ at $\q$, evaluated at $\p$:
    \[ \ell(\q;\p) = H_{\ell}(\q) + d H_{\ell}(\q) \cdot (\p - \q) . \]
\end{corollary}

\begin{corollary} \label{cor:diff-loss-bregman}
  When the true distribution is $p$, the improvement in expected loss for reporting $\p$ instead of $\q$ is the divergence function of $H_{\ell}$ (the Bregman divergence of $-H_{\ell}$), i.e.
    \[ \ell(\q;\p) - \ell(\p;\p) = D_{-H_{\ell}}(\p,\q) . \]
\end{corollary}

\begin{example}
  Recall from Example \ref{ex:log-loss-diff-kl} the log loss $\ell(\q,x) = \ln \tfrac{1}{\q}$ has expected loss equal to Shannon entropy.
  The associated Bregman divergence is the KL-divergence, so the difference in expected log loss between $\p$ and $\q$ under true distribution $\p$ is $KL(\p,\q) \eqdef \sum_x \px \ln \tfrac{\px}{\qx}$.
  The quadratic loss has expected loss $H_{\text{quad}}(\p) = \frac{1}{2} - \tfrac{1}{2}\|\p\|_2^2$, so the associated Bregman divergence is $D_{-H_{\text{quad}}}(\p,\q) = \tfrac{1}{2}\|\p-\q\|_2^2$.
\end{example}

The above are all well-known, although in the literature on proper scoring rules everything is negated (a score is used equal to negative loss, the expected score is convex, etc.).

\subsection{Strongly concave functions and strong properness}
Given the above characterization and our (carefully chosen) definition of \emph{strongly proper}, the classic characterization of proper losses extends easily:
\begin{theorem} \label{thm:proper-sc-char}
  A proper loss function $\ell$ is $\beta$-strongly proper (with respect to a norm) if and only if $H_{\ell}$ is $\beta$-strongly concave (with respect to that norm).
\end{theorem}
\begin{proof}
  We have $\ell(\q;\p) - \ell(\p;\p) = D_{-H_{\ell}}(\p,\q)$ by Corollary \ref{cor:diff-loss-bregman}.
  $H_{\ell}$ is $\beta$-strongly concave if and only if $D_{-H_{\ell}}(\p,\q) \geq \frac{\beta}{2}\norm{\p-\q}$ for all $\p,\q$, which is the condition that $\ell$ is $\beta$-strongly proper.
\end{proof}
Though the proof is trivial once the definitions are set up and followed through, the statement is powerful.
It completely characterizes the proper loss functions satisfying that, if $\q$ is significantly wrong (far from $\p$), then its expected loss is significantly worse.
It also gives an immediate recipe for constructing such losses: Start with any concave function $H(\q)$ that is strongly concave in your norm of choice, and set $\ell(\q,x) = H(\q) + dH(\q) \cdot (\pointx - \q)$.
All strongly proper losses satisfy this construction for some such $H$.

\subsection{Known examples}
Recall that the log scoring rule's expected loss function is Shannon entropy.
Hence, the fact that log loss is $1$-strongly-proper (Example \ref{ex:strongly-proper}) turns out to be equivalent to the statement that Shannon entropy is $1$-strongly convex in $\ell_1$ norm.
As described in Section \ref{sec:criteria}, this fact (perhaps surprisingly) is equivalent to Pinsker's inequality.

However, $\ell_1$-strong properness seems difficult to satisfy over the simplex.
In particular,
\begin{proposition}
  The quadratic scoring rule is not strongly proper in $\ell_1$ norm.
\end{proposition}
\begin{proof}
  Consider $\q$ as the uniform distribution and let $\px \in \frac{1 \pm \epsilon}{N}$, such that $\norm{\p - \q}_1 = \epsilon$.
  Then $\ell(q;p) - \ell(p;p) = \frac{1}{2}\norm{\p-\q}_2^2 = \frac{1}{2}(N)\left(\frac{\epsilon}{N}\right)^2 = \frac{\epsilon^2}{2N}$.
  As $N \to \infty$, this difference in loss goes to zero while $\norm{\p-\q}_1 = \epsilon$, so there is no fixed $\beta$ such that the loss is $\beta$-strongly proper.
\end{proof}

We can show that it is strongly proper in $\ell_2$ norm.
However, the usefulness of $\ell_2$ strong properness is less clear, as is demonstrated in Appendix \ref{app:strong-l2}.
\begin{lemma}
  The function $-\frac{1}{2}\|\p\|_2^2$ is $1$-strongly concave with respect to the $\ell_2$ norm.
\end{lemma}
\begin{proof}
  The associated Bregman divergence is $\frac{1}{2}\|\p-\q\|_2^2$, so it is $1$-strongly convex in $\ell_2$ norm.
\end{proof}
\begin{corollary} \label{cor:quadloss-sp-l2}
  The quadratic loss is $1$-strongly proper with respect to the $\ell_2$ norm.
\end{corollary}

\subsection{New proper losses}
Because the $\ell_1$ norm is especially preferred when measuring distances between probability distributions, we seek losses that are $1$-strongly proper with respect to the $L_1$ norm.
By the characterization of Theorem \ref{thm:proper-sc-char}, this is equivalent to seeking $\ell_1$ $\beta$-strongly-convex functions of probability distributions.

\begin{lemma} \label{lemma:sc-hessian}
  Let $M(x) \in \R^{N \times N}$ be the negative of the Hessian of a function $H_{\ell}: \Delta_{\X} \to \R$.
  Then $H_{\ell}$ is $\beta$-strongly concave in $\ell_1$ norm if, for all $x,w \in \R^n$,
    \[ w^{\intercal} M(x) w \geq \beta \|w\|_1^2  . \]
\end{lemma}
\begin{proof}
  Following e.g. \cite{boyd2004convex}, given any $x$ and $w \neq 0$, there exists an $\alpha \in [0,1]$ and $z = x + \alpha w$ such that
  \begin{align*}
    H_{\ell}(x+w) &= H_{\ell}(x) + \nabla H_{\ell}(x) \cdot w - \frac{1}{2} w^{\intercal} M(z) w  \\
                  &\leq H_{\ell}(x) + \nabla H_{\ell}(x) \cdot w - \frac{\beta}{2} \|w\|_1^2 .
  \end{align*}
\end{proof}

We focus on \emph{separable, symmetric} concave functions: $H(\q) = \sum_x h(\qx)$ for some concave function $h$.
In this case the Hessian of $H$ is a diagonal matrix with $(x,x)$ entry $\frac{d^2 h(z)}{dz^2}$.
Call its negative $M$ as in Lemma \ref{lemma:sc-hessian} and for convenience later, let us define $f(z)$ as
  \[ \frac{1}{f(z)} \eqdef \frac{- d^2 h(z)}{dz^2} . \]
Then by Lemma \ref{lemma:sc-hessian}, $H(\q)$ is $\beta$-strongly concave if
\begin{align*}
  \beta &\leq \min_{w: \|w\|_1=1} w^{\intercal} M w \\
        &=    \min_{w: \|w\|_1=1} \sum_x \frac{w_x^2}{f(\qx)}.
\end{align*}
This is solved by setting $w_x \propto f(\qx)$, where the normalizing constant is $C \eqdef \sum_x f(\qx)$.
So we have
\begin{align*}
  \beta &\leq \sum_x \left(\frac{f(\qx)}{C}\right)^2 \frac{1}{f(\qx)}  \\
        &=    \frac{1}{C^2} \sum_x f(\qx)  \\
        &=    \frac{1}{C}  \\
        &=    \frac{1}{\sum_x f(\qx)} .
\end{align*}
So for $1$-strong concavity, we require $\sum_x f(\qx) \leq 1$ for all $\q$.
Now choose $f(\qx) = \qx^{1 + \alpha}$.
\begin{itemize}
  \item If $\alpha < 0$, then $\sum_x f(\qx)$ can be arbitrarily large and the resulting function is not strongly concave in $\ell_1$ norm.
  \item If $\alpha = 0$, then we have $\frac{d^2 h(z)}{dz^2} = \frac{-1}{z}$ and we recover $h(z) = z \ln(\tfrac{1}{z})$, which gives $H$ as Shannon entropy; the log scoring rule.
  \item If $\alpha \geq 1$, we get $h(z)$ is unbounded on $[0,1]$, so we obtain an expected loss function that is unbounded on the simplex.
  \item For $0 < \alpha < 1$, we get a class of apparently-new proper loss functions that are $1$-strongly proper.
        Here $\frac{d^2 h(z)}{dz^2} = \frac{-1}{z^{1+\alpha}}$, so $h(z) = z^{1-\alpha}$ and $H(z) = \sum_x \qx^{1-\alpha}$.
\end{itemize}
In particular, for the last class, we identify the appealing case $\alpha = 0.5$.
It gives the following ``inverse root'' loss function:
\begin{itemize}
  \item $H(\q) = 2 \sum_x \sqrt{\qx}$.
  \item $\ell(\q,x) = \frac{1}{\sqrt{\qx}} + \sum_{x'} \sqrt{\qat{x'}}$.
  \item $\ell(\q;\p) = \sum_x \frac{1}{\sqrt{\qx}} \left(\px + \qx\right)$.
  \item $D_{-H}(\p,\q) = \sum_x \frac{1}{\sqrt{\qx}} \left(\sqrt{\px} - \sqrt{\qx}\right)^2$.
\end{itemize}
We are not aware of this loss having been used before, but it seems to have nice properties. There is an apparent similarity to the squared Hellinger distance $H(\p,\q)^2 \eqdef \frac{1}{2}\sum_x \left(\sqrt{\px} - \sqrt{\qx}\right)^2$, but we are not aware of a closer formal connection.
    For example, Hellinger distance is symmeteric.

\end{document}